\PassOptionsToPackage{dvipsnames}{xcolor} 
\documentclass[11pt]{article}
\usepackage{mathrsfs,amsmath,amsfonts,amssymb,bm,bbm,dsfont,pifont,amscd,stmaryrd,euscript,amsthm,color,epsfig,xr,yfonts,tikz,verbatim}

\usepackage{tikz-cd} 
\usepackage{include/shortcutsg}
\usepackage{authblk}
\usepackage[utf8]{inputenc} 
\usepackage[T1]{fontenc}    
\usepackage{url}            
\usepackage{booktabs}       
\usepackage{amsfonts}  
\usepackage{amsmath} 
\usepackage{nicefrac}       
\usepackage{microtype}      
\usepackage{xcolor}
\usepackage{graphicx}
\usepackage{array}
\usepackage{subcaption,mathrsfs}
\usepackage[sort]{natbib}
\usepackage{wrapfig}
\usepackage{tikz}
\usepackage{include/titletoc}
\usepackage[makeroom]{cancel}
\usepackage{algorithm}
\usepackage{algpseudocode}

\usepackage[dvipsnames]{xcolor}
\definecolor{dgreen}{rgb}{0.00,0.49,0.00}
\definecolor{dblue}{rgb}{0,0.08,0.75}
\usepackage{hyperref}       
\hypersetup{
    colorlinks =true,
    citecolor=dgreen,
    urlcolor=dblue,
    linkcolor=dblue
}
\usepackage[capitalize]{cleveref}

\newtheorem{thm}{Theorem}
\newtheorem{lem}{Lemma}

\newtheorem{cor}{Corollary}
\newtheorem{cond}{Condition}

\theoremstyle{definition}
\newtheorem{rem}{Remark}
\newtheorem{defn}{Definition}
\newtheorem{assum}{Assumption}

\Crefname{assum}{Assumption}{Assumptions}
\Crefname{thm}{Theorem}{Theorems}
\Crefname{lem}{Lemma}{Lemmas}
\Crefname{cor}{Corollary}{Corollaries}
\Crefname{cond}{Condition}{Conditions}
\Crefname{rem}{Remark}{Remarks}

\DeclareMathOperator{\E}{\mathbb{E}}
\DeclareMathOperator{\Var}{Var}

\newcommand{\Prb}{\mathbb{P}}

\newcommand{\cF}{\mathcal{F}}

\DeclareMathOperator*{\argmax}{arg\,max}
\DeclareMathOperator*{\argmin}{arg\,min}

\newcommand{\lra}[1]{\left\langle #1 \right\rangle}

\let\hat\widehat

\usepackage{enumitem}

\usepackage{enumitem}

\title{{\bf Nonparametric Instrumental Variable Analysis Without Structural Equations}:\\ {\Large\it Debiased Inference on Functionals of Inverse Problems with No Solutions}}

\author[1]{Zikai Shen}
\author[2,3]{Nathan Kallus}
\author[4]{Dimitri Meunier}
\author[4]{Houssam Zenati}
\author[4]{Arthur Gretton}
\author[2]{Aur\'elien Bibaut}

\affil[1]{University College London}
\affil[2]{Netflix}
\affil[3]{Cornell University}
\affil[4]{Gatsby Computational Neuroscience Unit, University College London}

\begin{document}

\renewcommand{\doteq}{:=}

\maketitle

\begin{abstract}
We consider debiased inference on finite-dimensional functionals of infinite-dimensional least-squares solutions to inverse problems as a way to avoid having to assume exact solutions exist. Such assumptions are substantive and not innocuous, and their failure may imperil inference when we impose them on the statistical model. Our approach instead allows us to conduct inference on a quantity that is defined regardless of solutions existing and coincides with the usual estimands when they do. For the case of instrumental variables, this means we can motivate the analysis with structural models but these do not need to hold exactly for the semiparametric inferential procedure to remain valid.
\end{abstract}

\section{Introduction}

Instrumental variable (IV) analyses generally start by posing a \textit{structural equation}:
\begin{equation}Y=h_{\rm structural}(X)+\epsilon,\label{eq:structural}\end{equation}
where $h_{\rm structural}$ represents the \textit{causal effect} of $X$ on $Y$, and $X$ and $\epsilon$ may be endogenous ($\E[\epsilon\mid X]\neq0$).
Then given an exogenous instrument $Z$ satisfying the exclusion restriction, the common \textit{statistical} solution given joint observations of $W=(X,Y,Z)\sim P$ is to conduct inference on some continuous linear functional $h\mapsto \E_P[m(W;h)]$
of a solution $h\in\Hcal$ to the linear equation implied by exclusion:
\begin{equation}\label{eq:CMR}
T_P h=r_P,
\end{equation}
where $T_P:\Hcal\to\Gcal$ maps $h\mapsto \argmin_{g\in \Gcal} \E_P(h(X)-g(Z))^2$, $r_P=\argmin_{r\in \Gcal} \E_P(Y-r(Z))^2$, and $\Hcal$, $\Gcal$ are closed linear subspaces of square-integrable functions of $X$ and of $Z$, respectively. For example, if these are \textit{all} square-integrable functions, then $(T_Ph)(Z)=\E_P[h(X)\mid Z]$ is the conditional expectation.

There are two possible challenges with this:
\begin{itemize}
\item \textbf{The existence of solutions to \cref{eq:CMR}}. When we start with \cref{eq:structural} with $h_{\rm structural}\in\Hcal$ we are guaranteed a solution to \cref{eq:CMR} in $\Hcal$ by construction. For example, we may assume a \textit{linear} structural equation \citep{theil1953repeated,basmann1957generalized,sargan1958estimation,anderson2005origins}. Nonparametric IV (NPIV) analysis aims to relax this by allowing $h_{\rm structural}$ to be \textit{any} (square-integrable) function. However, this still assumes some solution exists, which is still a very substantive assumption. If we impose this assumption on the statistical model we consider, as done in \citet{severini2012efficiency,bennett2023sourceconditiondoublerobust,bennett2025inference,ai2012semiparametric,ai2003efficient,chen2009efficient,chen2012estimation,kallus2021causal,cui2024semiparametric} who achieve efficiency in this statistical model, then the estimator may well deteriorate very ungracefully when no perfect solution to \cref{eq:CMR} actually exists. This can be understood as exclusion and/or homogeneity (additive $\epsilon$) not always holding exactly, as may occur in practice. To illustrate the point: in an ``overidentified" setting with many instruments, would we really believe all instruments perfectly identify the \emph{exact same} causal effect \citep{hall2003large,andrews2019structure,andrews2025purpose}?
\item \textbf{The uniqueness of $\E_P[m(W;h)]$ over solutions}. Even if we have solutions to \cref{eq:CMR}, there may be many, and they may not all lead to the same value of $\E_P[m(W;h)]$. This can be understood as instrument weakness: the variations $Z$ induces in $X$ can be insufficient to characterize the aspects of $h_{\rm structural}$ we care about. That $T_Ph=T_Ph'$ implies $\E_P[m(W;h)] = \E_P[m(W;h')]$ for all $h,h'\in\Hcal$ would be implied by the existence of a solution in $\Gcal$ to $T_P^*g={a_P}$, where $a_P\in\Hcal$ is the Riesz representer of $h\mapsto \E_P[m(W;h)]$ (since this implies $a_P\in\range(T_P^*)\subseteq\kernel(T_P)^\bot$). These identification assumptions have been explored in \citet{severini2012efficiency,ichimura2022influence, bennett2023sourceconditiondoublerobust,bennett2025inference, vanderlaan2026nonparametricinstrumentalvariableinference,10.1257/aer.20231765}. However, the same conundrum arises: a solution to that linear equation may well only \textit{approximately} exist in practice.
\end{itemize}

In both cases, it would behoove us to not impose the assumptions of the existence of solutions (to $T_Ph=r_P$, to $T_P^*g=a_P$, or to both) on any statistical model. While the equations may be well motivated, \textit{perfect} solutions may not really exist in practice and this can derail inference procedures that assume they do. Instead, we would rather target a statistical quantity that is well-defined \textit{regardless} of existence of such solutions, coincides with the structural quantities when they do make sense, and degrades reasonably in mild violations of existence of solutions.

Toward that end, in this paper we consider debiased inference on the following estimand:
 \begin{align*}
     \Psi(P)&\doteq \mathbb{E}_P[g(Z)h(X)],\quad h\in \Hcal_P,\quad g\in\Gcal_P,
 \end{align*}
 where $\Hcal_P\doteq \argmin_{h\in\Hcal}\|T_P h-r_P\|_{L_2(P_Z)}^2$ and $\Gcal_P\doteq\argmin_{g\in\Gcal}\|T_P^* g-a_P\|_{L_2(P_X)}^2$, which, under a mild assumption given in the next section (\Cref{assum: range_condition}), is well defined without requiring the existence of solutions to either equation $T_P h =r_P$ or $T_P^{\ast} g = a_P$, and regardless of which minimizer we choose in the above population least-squares problems. Of course, when solutions to $T_Ph = r_P$ and $T_P^{\ast}g = a_P$ exist, we always recover the original causal structural estimand.

To allow sufficient generality so as to also cover such related problems as proximal causal inference \citep{miao2018identifying, deaner2018proxy, cui2024semiparametric,kallus2021causal}, we generalize $r_P$ to be the (unknown) Riesz representer of some continuous linear functional
\begin{align}
    \label{eq:m_tilde}
    \Gcal\to\RR:g\mapsto\E_P[\tilde m(W;g)],
\end{align}
where $\tilde m$ is a known function. For example, for the IV setting, we set $\tilde m(W;g)=Yg(Z)$. Similarly, we also let $a_P$ be unknown, and only require that it is the (unknown) Riesz representer of some continuous linear functional $\Hcal\to\RR:h\mapsto\E_P[m(W;h)]$, where $m$ is a known function. We refer the reader to Examples 2 and 3 in \citet{bennett2025inference} for instantiation in proximal causal inference and missing-not-at-random data.


\section{Setup and Identification}


\paragraph{Notation.} Throughout this paper we consider three random variables $X,Y,Z$ which take values respectively in the separable measurable spaces $E_X, E_Y, E_Z$ endowed with their respective Borel $\sigma$-fields $\mathcal{F}_{E_X}, \mathcal{F}_{E_Y}, \mathcal{F}_{E_Z}$. We write $W$ for the random vector $(X,Y,Z)$, and we let $E_W = E_X \times E_Y \times E_Z$. We let $(\Omega, \mathcal{F}, P)$ denote the underlying probability space with expectation operator $\mathbb{E}_P$, abbreviated as $\mathbb{E}$. Let $P_X, P_Y, P_Z$ denote the pushforward of $P$ under $X,Y,Z$ respectively, i.e. $X\sim P_X, Y\sim P_Y, Z\sim P_Z$. We denote the space of real-valued Lebesgue square integrable functions on $(E_X, \mathcal{F}_{E_X})$ with respect to $P_X$ as $L^2(E_X, \mathcal{F}_{E_X}, P_X)$, abbreviated as $L^2(P_X)$. We define $L^2(P_Y)$ and $L^2(P_Z)$ similarly. We let $\Hcal, \Gcal$ denote closed linear subspaces of $L^2(P_X)$ and $L^2(P_Z)$ respectively. We write $\|\cdot\|_{\Hcal} = \|\cdot\|_{L^2(P_X)}$ and $\|\cdot\|_{\Gcal} = \|\cdot\|_{L^2(P_Z)}$. Let $\Hcal_1$ and $\Hcal_2$ be Hilbert spaces. We write $\mathcal{L}(\Hcal_1,\Hcal_2)$ as the Banach space of bounded linear operators from $\Hcal_1$ to $\Hcal_2$, equipped with the operator norm $\|\cdot\|_{\mathrm{op}}$. For an operator $A\in \mathcal{L}(\Hcal_1, \Hcal_2)$, we let $A^{\ast} \in \mathcal{L}(\Hcal_2, \Hcal_1)$ denote its adjoint. For a subset $A\subseteq \Hcal_1$, we let $\overline{A}$ denote its closure in $\Hcal_1$, and we let $A^{\perp} \subseteq \Hcal_1$ denote its orthogonal complement in $\Hcal_1$. 

We let $T_P: \Hcal\to \Gcal$ be defined as the mapping $h\mapsto \argmin_{g\in \mathcal{G}} \mathbb{E}_P[(h(X) - g(Z))^2]$. $T_P$ is well-defined by the unique existence of orthogonal projection onto a closed linear subspace of a Hilbert space. Hence the operator $T_P$ is linear and bounded, with $\|T_P\|_{\mathrm{op}}\le 1$. For any $h\in \Hcal$ and $g\in \Gcal$, we have $\mathbb{E}_P[h(X)g(Z)] = \langle T_Ph, g\rangle_{L^2(P_Z)} = \langle h, T_P^{\ast}g \rangle_{L^2(P_X)}$, by properties of orthogonal projection. 

Let $X_n, X$ be random variables, for $n\in \mathbb{N}$. $X_n\xrightarrow{p}X$ denotes $X_n$ converges to $X$ in probability. $X_n \rightsquigarrow X$ denotes $X_n$ converges to $X$ in distribution. For two sequences of real numbers $(a_n), (b_n)$, $a_n=O(b_n)$ means $|a_n|/|b_n|$ is bounded, $a_n=o(b_n)$ means $|a_n|/|b_n|\to0$, $a_n=\omega(b_n)$ means $|a_n|/|b_n|\to\infty$, $X_n=O_p(b_n)$ means $X_n/b_n$ is bounded in probability, $X_n=o_p(b_n)$ means $X_n/b_n\xrightarrow{p}0$.    

\paragraph{Definition and well-definedness of target functional.}
We let $m: E_W \times \Hcal \to \mathbb{R}$ and $\widetilde{m} : E_W \times \Gcal \to \mathbb{R}$ denote two known functions, which are respectively linear in $\Hcal$ and $\Gcal$. We also require that the $\mathbb{R}$-valued linear functionals on $\Hcal$ and $\Gcal$ given by $h\mapsto \mathbb{E}_P[m(W;h)]$ and $g\mapsto \mathbb{E}_P[\widetilde{m}(W;g)]$ are both continuous. In \Cref{assum:mscont} in \Cref{sec:minimax_learners}, we will also require that they are continuous in a mean-squared sense. We let $a_P\in \Hcal$ and $r_P\in \Gcal$ denote their unique Riesz representers, whose existence is guaranteed by the Riesz-Fréchet representation theorem. We define the subsets 
\begin{align}
\Hcal_P&\doteq \argmin_{h\in\Hcal}\|T_P h-r_P\|_{L_2(P_Z)}^2, \label{eq:hcal_P}\\
\Gcal_P&\doteq \argmin_{g\in\Gcal}\|T_P^* g-a_P\|_{L_2(P_X)}^2.  \label{eq:gcal_P}
\end{align}
In this paper, we refer to the linear equation $T_P h = r_P$ as the \emph{primary} linear equation and $T_P^{\ast}g = a_P$ as the \emph{secondary} linear equation. We refer to an element of $\Hcal_P$ as \emph{primary} nuisance and an element of $\Gcal_P$ as \emph{secondary} nuisance. \footnote{This distinction is arbitrary as our setup is symmetric for $\Hcal_P$ and $\Gcal_P$.} We now state an assumption that is the necessary and sufficient condition for $\Hcal_P$ and $\Gcal_P$ to be non-empty. We first introduce some additional notation. Since $T_P$ and its adjoint operator $T_P^{\ast}$ enjoy the relation $\kernel(T_P)^{\perp} = \overline{\range(T_P^{\ast})}$, we have
\begin{align*}
    \Hcal &= \kernel(T_P)\oplus \kernel(T_P)^\perp = \kernel(T_P)\oplus \overline{\range(T_P^*)},\\
    \Gcal &= \kernel(T_P^*)\oplus \kernel(T_P^*)^\perp = \kernel(T_P^*)\oplus \overline{\range(T_P)}. 
\end{align*}
Accordingly, write
\begin{align*}
    &a_P =a_{P,\parallel}+a_{P,\perp},
    \qquad
    a_{P,\parallel}\in \kernel(T_P)^\perp
    = \overline{\range(T_P^*)},
    \qquad
    a_{P,\perp}\in \kernel(T_P),\\
     &r_P=r_{P,\parallel}+r_{P,\perp},\;
    \qquad
    r_{P,\parallel}\in \kernel(T_P^*)^\perp
    = \overline{\range(T_P)},
    \qquad
    r_{P,\perp}\in \kernel(T_P^*).
\end{align*}
The orthogonal residuals $r_{P, \perp}$ and $a_{P, \perp}$ will play a key role in our story. 
\begin{assum} \label{assum: range_condition} $r_{P,\parallel}\in \range(T_P), \; a_{P,\parallel}\in \range(T_P^*)$. 
\end{assum}
\begin{lem}[Nonemptiness of $\Hcal_P$ and $\Gcal_P$]\label{lem: nonempty_ls}We have $\Hcal_P\neq\varnothing$ if and only if $r_{P,\parallel}\in\range(T_P)$, and $\Gcal_P\neq\varnothing$ if and only if $a_{P,\parallel}\in\range(T_P^*)$. Hence \Cref{assum: range_condition} is necessary and sufficient for $\Hcal_P$ and $\Gcal_P$ to be nonempty. \end{lem} \begin{proof} Since $T_Ph\in\range(T_P)\subseteq\kernel(T_P^*)^\perp$ and $r_{P,\perp}\in\kernel(T_P^*)$, we have
\begin{align*}
    \|T_Ph-r_P\|_{L^2(P_Z)}^2=\|T_Ph-r_{P,\parallel}\|_{L^2(P_Z)}^2+\|r_{P,\perp}\|_{L^2(P_Z)}^2. 
\end{align*}
Therefore $\Hcal_P\neq\varnothing$ if and only if $r_{P,\parallel}\in\range(T_P)$. The argument for $\Gcal_P$ is identical. \end{proof}
\begin{rem}[Existence of solutions vs. least squares minimizers]
\label{rem:existence_solns_ls}
\Cref{assum: range_condition} asserts that the component of $r_P$ in $\kernel(T_P^{\ast})^{\perp}$ is in fact sufficiently smooth with respect to $T_P$, and it makes the analogous assertion for $a_P$. Equivalently,
\[
    r_P\in \kernel(T_P^*)+\range(T_P),
    \qquad
    a_P\in \range(T_P^*)+\kernel(T_P).
\]
This condition guarantees that the least squares solution sets $\Hcal_P$ and $\Gcal_P$ are nonempty. It automatically holds when $\mathcal H$ and $\mathcal G$ are finite-dimensional, since the closure of a finite dimensional linear space is always itself. Importantly, \Cref{assum: range_condition} does not impose solvability of
the primary or dual equations. Solvability of $T_Ph=r_P$ would require
both
\[
    r_{P,\parallel}\in \range(T_P)
    \qquad\text{and}\qquad
    r_{P,\perp}=0.
\]
Similarly, solvability of $T_P^*g=a_P$ would require both
\[
    a_{P,\parallel}\in \range(T_P^*)
    \qquad\text{and}\qquad
    a_{P,\perp}=0.
\]
Thus, compared with the usual identification condition in
\citet{bennett2023sourceconditiondoublerobust,bennett2025inference,severini2012efficiency},
which restricts $a_P$ to lie in $\range(T_P^*)$, we further allow $a_P$ to have a component in
$\kernel(T_P)$. Analogously, we allow $r_P$ to have a component in
$\kernel(T_P^*)$. In other words, our analysis does not presuppose existence of solutions to either
$T_Ph=r_P$ or $T_P^*g=a_P$. In \Cref{rem:bennett_exact_solution}, we see that we do indeed recover the results in \citet{bennett2023sourceconditiondoublerobust} when we assume both $r_{P,\perp} = 0$ and $a_{P, \perp} = 0$.
\end{rem}

We now show that $\mathbb E_P[h(X)g(Z)]$ is independent of the choice of least squares
solutions $h\in\mathcal H_P$ and $g\in\mathcal G_P$. 
\begin{lem}[Well-definedness of the target functional]
\label{lem: param_inv}
Under \Cref{assum: range_condition}, let $h,g$ denote any element of $\Hcal_P$ and $\Gcal_P$ respectively. We have $\mathcal{H}_P = h + \kernel(T_P)$ and $\mathcal{G}_P = g + \kernel(T_P^{\ast})$. 
    Hence, for any $h,h'\in \mathcal{H}_P$ and $g,g'\in \mathcal{G}_P$, we have
    \begin{align*}
        \mathbb{E}_P[h(X)g(Z)] = \mathbb{E}_P[h'(X)g'(Z)].
    \end{align*}
\end{lem}
\begin{proof}
Note that $h \in \mathcal{H}_P$ satisfies the normal equation
    \begin{align*}
        T_P^{\ast}T_Ph - T_P^{\ast}r_P = 0.
    \end{align*}
    We therefore have $\mathcal{H}_P = h + \kernel(T_P^{\ast}T_P) = h + \kernel(T_P)$. The statement for $\mathcal{G}_P$ follows analogously. For any $h,h'\in \mathcal{H}_P$ and $g,g'\in \mathcal{G}_P$, we have
    \begin{align*}
        &\mathbb{E}_P[h(X)g(Z)] = \mathbb{E}_P[h(X)(T_P^{\ast}g)(X)] = \mathbb{E}_P[h(X)(T_P^{\ast}g')(X)]\\ &= \mathbb{E}_P[h(X)g'(Z)] = \mathbb{E}_P[(T_P h)(Z)g'(Z)] = \mathbb{E}_P[h'(X)g'(Z)]. \qedhere
    \end{align*}
\end{proof}
Under \Cref{assum: range_condition} and by \Cref{lem: param_inv}, the following object is well-defined:
\begin{align}
    \Psi(P)&\doteq \mathbb{E}_P[g(Z)h(X)],\quad h\in \Hcal_P,\quad g\in\Gcal_P, \label{eq:target_functional}
\end{align}
We refer to $\Psi$ as the target functional, and we always make \Cref{assum: range_condition} for the rest of the paper. We observe that $\Psi$ is a \emph{bilinear} functional in the nuisances $g$ and $h$. 


\section{Functional bias expansion}
\label{sec: functional_bias}

\paragraph{Minimum norm elements of $\Hcal_P, \Gcal_P$.} Under \Cref{assum: range_condition}, the least-squares solution sets
$\Hcal_P$ and $\Gcal_P$ are nonempty closed affine subspaces. So,
we can define the minimum norm elements of $\Hcal_P, \Gcal_P$, i.e. minimum norm least squares solutions: 
\begin{align*}
    h^{\dag}_P = \argmin_{h\in \mathcal{H}_P} \|h\|_{L^2(P_X)},\quad  g^{\dag}_P = \argmin_{g\in \mathcal{G}_P} \|g\|_{L^2(P_Z)},
\end{align*}
where $\Hcal_P$ and $\Gcal_P$ are defined in \cref{eq:hcal_P} and \cref{eq:gcal_P}. Minimum norm elements provide canonical choices in $\Hcal_P$ and $\Gcal_P$ when these sets are non-empty. We let $T_P^{\dag}$ denote the Moore-Penrose pseudoinverse of $T_P$, which is a (generally unbounded) operator with domain $\mathrm{dom}(T_P^{\dag}) = \range(T_P) \oplus  \kernel(T_P^{\ast})$ \citep{engl1996regularization}. By definition of $\mathcal{H}_P$ and the fact that $r_P\in \mathrm{dom}(T_P^{\dag})$, we have $ h^{\dag}_P = T_P^{\dag} r_P$. The restricted operator $T_P|_{\kernel(T_P)^\perp}:\kernel(T_P)^\perp\to\range(T_P)$ is a bijection, and thus it has a well-defined inverse on $\range(T_P)$. Since $T_P^{\dag}$ annihilates the $\kernel(T_P^{\ast})$-component of $r_P$, we have
\begin{align*}
    h^{\dag}_P = T_P^{\dag} r_{P} = T_P^{\dag} r_{P,\parallel} = (T_P|_{\kernel(T_P)^{\perp}})^{-1}r_{P,\parallel}.
\end{align*}
In addition, the pseudoinverse can also be interpreted as the Tikhonov regularized solution in the limit of vanishing regularization,
\begin{align}
    h^{\dag}_P = \lim_{\lambda \downarrow 0} h_{P}^{\lambda}, \quad h_P^{\lambda}= \argmin_{h\in \mathcal{H}}\lambda \|h\|^2_{L^2(P_X)} + \|T_Ph - r_P\|_{L^2(P_Z)}^2. \label{eqn:populatoin_objective}
\end{align}
Symmetrically, we have
\begin{align*}
    g^{\dag}_P = \argmin_{g \in \mathcal{G}_P}\|g\|_{L^2(P_Z)} = (T_P^{\ast})^{\dag}a_P = (T_P^{\ast})^{\dag}a_{P,\parallel} = \left(T_P^{\ast}|_{\kernel(T_P^{\ast})^{\perp}}\right)^{-1}a_{P,\parallel},
\end{align*}
where the last equality holds because $(T_P^{\ast})^{\dag}$ annihilates the $\kernel(T_P)$-component of $a_P$. Thus
\[
    T_Ph_P^\dag=r_{P,\parallel},
    \qquad
    T_P^*g_P^\dag=a_{P,\parallel},
    \qquad
    \Psi(P)=\E_P[h_P^\dag(X)g_P^\dag(Z)].
\]
For these fixed targets, define respectively
\begin{align*}
    \Psi_{h_P^\dag}(g) &= \mathbb{E}_P[h_P^\dag(X)g(Z)], \quad \Psi_{h_P^\dag}: L^2(P_Z)\to\mathbb R,\\
    \Psi_{g_P^\dag}(h) &= \mathbb{E}_P[h(X)g_P^\dag(Z)], \quad \Psi_{g_P^\dag}: L^2(P_X)\to\mathbb R.
\end{align*}

\paragraph{Source condition.} We will next impose an analogous assumption to the \emph{source condition} in \citet{bennett2025inference,bennett2023sourceconditiondoublerobust} that controls the smoothness of our functionals and enables a functional bias expansion we can use for debiased estimation. However, unlike \citet{bennett2025inference,bennett2023sourceconditiondoublerobust}, consistent with not assuming that solutions to either the primary or secondary linear equations exist, we will impose source conditions only on the projected components and only as needed for the nonzero orthogonal residuals $r_{P, \perp}$ and $a_{P, \perp}$. Throughout this work, we always refer to norm induced by the projection operator $T_P$ or $T_P^{\ast}$ as \emph{weak metrics}, and $L^2$ norms as \emph{strong metrics}. This terminology is classical in NPIV and ill-posed inverse problem, as weak metrics are prediction norms, and unlike strong metrics, weak metric nuisance errors can be controlled without restrictive assumptions on measure of ill-posedness \citep{chen2011rate, deaner2018proxy}.

\begin{assum}[Existence of weak metric Riesz representers for nuisance where solution does not exist]\label{assum: source_condition_riesz}
We assume both of the following statements hold
\begin{align*}
    &r_{P,\perp}=0
    \quad\text{or}\quad
    a_{P,\parallel}\in\range(T_P^*T_P)\\
    &a_{P,\perp}=0
    \quad\text{or}\quad
    r_{P,\parallel}\in\range(T_PT_P^*).
\end{align*}
\end{assum}
\begin{rem}[\Cref{assum: range_condition} and \Cref{assum: source_condition_riesz} in the finite dimensional setting]
If $\mathcal{H}_P$ and $\mathcal{G}_P$ are finite dimensional, then $\range(T_P) = \overline{\range(T_P)},\,\range(T_P^{\ast}) = \overline{\range(T_P^{\ast})}$, therefore \Cref{assum: range_condition} holds automatically. Furthermore, we have $\range(T_PT_P^{\ast}) = \range(T_P),\,\range(T_P^{\ast}T_P) = \range(T_P^{\ast})$ by an application of rank-nullity theorem, therefore \Cref{assum: source_condition_riesz} also holds for free. Therefore, \Cref{assum: range_condition} and \Cref{assum: source_condition_riesz} do not constrain the set of allowed distributions $P$ in the finite dimensional setting. 
\end{rem}

\begin{rem}[Interpretation of \Cref{assum: source_condition_riesz}]
    As explained in \Cref{rem:existence_solns_ls}, assuming $r_{P, \perp} = 0$ and $r_{P, \parallel}\in \range(T_P)$ (as we assume in \Cref{assum: range_condition}) is equivalent to the existence of a solution to $T_P h = r_P$. If a solution exists, then we do \emph{not} need to assume a source condition $a_{P, \parallel} \in \range(T_P^{\ast}T_P)$. Indeed, we will show in \Cref{lem:Riesz_representer_equiv_assum} that $a_{P,\parallel}\in \range(T_P^{\ast}T_P)$ is equivalent to assuming the secondary nuisance $g_P^{\dag}$ lies in the range of $T_P$. As we explain in \Cref{rem:fbias_alpha_nuisance_discussion} after \Cref{thm: bias_char}, this range condition is precisely what enables us to debias an estimated non-zero residual $r_{P, \perp}$, when it is necessary to perform this estimation. The same story applies to the other side, where if we wish to avoid assuming the existence of a solution to $T_P^{\ast}g = a_P$, then we must estimate the non-zero residual $a_{P, \perp}$ that lies orthogonal to the range of $T_P^{\ast}$. To debias the estimation error for $a_{P, \perp}$, we must impose a source condition, or equivalently a range condition for the primary nuisance $h_P^{\dag}$. We offer additional discussion in \Cref{rem:bennett_exact_solution}.  
\end{rem}

Note that \Cref{assum: source_condition_riesz} is stronger than \Cref{assum: range_condition} only on sides where the corresponding orthogonal residual is nonzero. \Cref{assum: range_condition} should be understood as the minimal assumption to ensure the existence of least-squares solutions. \Cref{assum: source_condition_riesz} additionally requires existence of solutions to Riesz regression \citep{chernozhukov2024automaticdebiasedmachinelearning} defined via the \emph{weak metric}, on sides where the residual $r_{P,\perp}$ or $a_{P, \perp}$ is non-zero. To make this concrete, we define the Riesz regression solution \emph{sets}
\begin{align}
    \mathcal A_P^h
    &=
    \argmin_{\alpha^h\in\Hcal}
    \left\{
    \frac12\|T_P\alpha^h\|_{L^2(P_Z)}^2
    -
    \langle a_{P,\parallel},\alpha^h\rangle_{L^2(P_X)}
    \right\},
    \label{eq: alpha_h_P_riesz_regression}\\
    \mathcal A_P^g
    &=
    \argmin_{\alpha^g\in\Gcal}
    \left\{
    \frac12\|T_P^*\alpha^g\|_{L^2(P_X)}^2
    -
    \langle r_{P,\parallel},\alpha^g\rangle_{L^2(P_Z)}
    \right\}.
    \label{eq: alpha_g_P_riesz_regression}
\end{align}
When these sets are nonempty, we define the minimum-norm elements
\begin{align*}
    \alpha_P^{h,\dag} = \argmin_{\alpha^h \in \mathcal{A}^h_P} \|\alpha^h\|_{L^2(P_X)},\quad \alpha_P^{g,\dag} = \argmin_{\alpha^g\in \mathcal{A}_P^{g}} \|\alpha^g\|_{L^2(P_Z)}. 
\end{align*}
Furthermore, the first-order conditions of these optimization problems yield
\begin{align*}
    \langle h, T_P^{\ast}T_P\alpha_P^{h}\rangle_{L^2(P_X)}
    &=
    \langle a_{P,\parallel}, h\rangle_{L^2(P_X)}
    =
    \mathbb{E}_P[h(X)g_P^\dag(Z)],
    \quad \forall h\in \mathcal{H},\\
    \langle g, T_PT_P^{\ast}\alpha_P^{g}\rangle_{L^2(P_Z)}
    &=
    \langle r_{P,\parallel}, g\rangle_{L^2(P_Z)}
    =
    \mathbb{E}_P[g(Z)h_P^\dag(X)],
    \quad \forall g\in \mathcal{G}.
\end{align*}

We next establish the equivalence between the projected source conditions and the existence of solutions to the Riesz regression problems and how these solutions give rise to the minimum-norm least-squares solutions. Our argument adapts Theorem 2 and Lemma 3 of \citet{bennett2025inference}.
\begin{lem}\label{lem:Riesz_representer_equiv_assum}
The following statements are equivalent:
\begin{enumerate}
    \item $a_{P,\parallel}\in\range(T_P^*T_P)$.
    \item $\mathcal A_P^h$ is nonempty.
    \item $g_P^\dag\in\range(T_P)$.
\end{enumerate}
When these conditions hold, every $\alpha_P^h\in\mathcal A_P^h$ satisfies
\begin{align}
    T_P^*T_P\alpha_P^h=a_{P,\parallel},
    \qquad
    T_P\alpha_P^h=g_P^\dag.
    \label{eq: T_P_gdag}
\end{align}
Similarly, the following statements are equivalent:
\begin{enumerate}
    \item $r_{P,\parallel}\in\range(T_PT_P^*)$.
    \item $\mathcal A_P^g$ is nonempty.
    \item $h_P^\dag\in\range(T_P^*)$.
\end{enumerate}
When these conditions hold, every $\alpha_P^g\in\mathcal A_P^g$ satisfies
\begin{align*}
    T_PT_P^*\alpha_P^g=r_{P,\parallel},
    \qquad
    T_P^*\alpha_P^g=h_P^\dag.
\end{align*}
\end{lem}

\begin{rem}[Source condition]
\label{rem: Riesz_representer_equiv_assum}
Let $\mathcal H_0=\kernel(T_P)^\perp$. Equip $\mathcal H_0$ with the inner product induced by the projection operator $T_P$
\begin{align*}
    \langle u,v\rangle_{T_P} \doteq \langle T_Pu,T_Pv\rangle_{L^2(P_Z)}.
\end{align*}
This defines an inner product on $\mathcal H_0$, but $\mathcal H_0$ need not be
complete under the norm $\|T_P(\cdot)\|_{L^2(P_Z)}$ unless $\range(T_P)$ is closed. This is generally false if $\range(T_P)$ is infinite dimensional, due to the ill-posedness of the inverse problem.

To remedy this, we use the fact that the pair $(\mathcal{H}_0, \|T_P(\cdot)\|_{L^2(P_Z)})$ is a pre-Hilbert space. We let $\overline{\mathcal H_0}^{\,T_P}$ denote the completion of this pre-Hilbert space, and the bilinear form $\langle \cdot, \cdot\rangle_{T_P}$ can be extended to an inner product on $\overline{\mathcal H_0}^{\,T_P}$.  The map $u\mapsto T_Pu$ is an isometry from $\mathcal H_0$ onto
$\range(T_P)$, and therefore extends \emph{uniquely} to an isometric isomorphism
\begin{align*}
    \overline{\mathcal H_0}^{\,T_P}\cong \overline{\range(T_P)}.
\end{align*}
Consider the linear functional $\ell_a:\mathcal H_0\to\mathbb R$
\begin{align*}
    \ell_a(u)=\langle a_{P,\parallel},u\rangle_{L^2(P_X)},
    \qquad u\in\mathcal H_0.
\end{align*}
Under \Cref{assum: range_condition}, $a_{P,\parallel}\in\range(T_P^*)$. Recall that
$g_P^\dag\in\kernel(T_P^*)^\perp=\overline{\range(T_P)}$ is the
minimum-norm solution of $T_P^*g=a_{P,\parallel}$. Then
\begin{align*}
    \ell_a(u)
    =
    \langle g_P^\dag,T_Pu\rangle_{L^2(P_Z)},
\end{align*}
so $\ell_a$ is continuous with respect to the weak metric. We can therefore extend $\ell_a$ to a continuous linear functional with respect to the Hilbert space $\overline{\mathcal H_0}^{\,T_P}$. Applying the Riesz-Fréchet representation theorem to $\ell_a$ defined on $\overline{\mathcal H_0}^{\,T_P}$, we know that $\ell_a$ admits a unique Riesz representer, given by the preimage of $g_P^\dag$
under the isometric isomorphism
$\overline{\mathcal H_0}^{\,T_P}\cong\overline{\range(T_P)}$.

The condition $a_{P,\parallel}\in\range(T_P^*T_P)$ is the
stronger requirement that the Riesz representer of $\ell_a$ in the Hilbert space $(\overline{\mathcal H_0}^{\,T_P}, \|T_P(\cdot)\|_{L^2(P_Z)})$ is
represented by an element of the original Hilbert space $(\mathcal H_0, \|\cdot\|_{\mathcal{H}})$.
If, in addition, $g_P^\dag\in\range(T_P)$, then there exists
$\alpha_P^h\in\mathcal H_0$ such that
\begin{align*}
    T_P\alpha_P^h=g_P^\dag.
\end{align*}
Applying $T_P^*$, this is equivalent to
\begin{align*}
    T_P^*T_P\alpha_P^h=a_{P,\parallel},
\end{align*}
that is, $a_{P,\parallel}\in\range(T_P^*T_P)$. The solution in $\kernel(T_P)^\perp$ is therefore the minimum-norm element of $\mathcal A_P^h$.  Our discussion applies symmetrically to $r_{P,\parallel}\in\range(T_PT_P^*)$ by exchanging the roles of $T_P$ and $T_P^{\ast}$. 
\end{rem}
\Cref{rem: Riesz_representer_equiv_assum} shows that the minimum-norm element of $\mathcal{A}_P^{h}$ is a Riesz representer with respect to the topology induced by $\|T_P(\cdot)\|_{L^2(P_Z)}$. This justifies our nomenclature of referring to \cref{eq: alpha_h_P_riesz_regression} and \cref{eq: alpha_g_P_riesz_regression} as Riesz regressions.

\begin{rem}[Comparison with \citet{bennett2025inference}]
    We compare Assumption~\ref{assum: source_condition_riesz} with the \emph{functional strong identification} condition of \citet[Definition 1]{bennett2025inference}. That condition can also be understood as the existence of a solution to a Riesz regression problem. However, unlike that work, we require existence of solutions to Riesz regressions only for linear functionals defined by $a_{P, \parallel}, r_{P, \parallel}$, and moreover only on sides where the corresponding orthogonal residual $r_{P, \perp}, a_{P, \perp}$ is non-zero. In contrast, \citet{bennett2025inference} requires existence of solutions to Riesz regression for linear functional defined by $a_P$ in our notation, which a priori constrains the statistical model.
\end{rem}
 
For a generic nuisance tuple
\[
    \eta=(h,g,\alpha^h,\alpha^g,r_\perp,a_\perp),
    \qquad
    h,\alpha^h\in\Hcal,\quad g,\alpha^g\in\Gcal,
    \qquad
    r_{\perp}\in \Gcal, a_{\perp}\in \Hcal,
\]
define
\begin{equation}
    \label{eq:rif}
    \begin{aligned}
    \chi(W;\eta)
    &:=
    m(W;h)+\widetilde m(W;g)-h(X)g(Z)\\
    &\quad+
    r_\perp(Z)\{\alpha^h(X)-g(Z)\}
    +
    a_\perp(X)\{\alpha^g(Z)-h(X)\}.
    \end{aligned}
\end{equation}
\begin{rem}[Special case where solutions to linear inverse problem(s) exist]\label{rem:bennett_exact_solution}
The setting of \citet{bennett2023sourceconditiondoublerobust,bennett2025inference}, where both linear inverse problems admit solutions, imposes both $r_{P,\perp}=0$ and $a_{P,\perp}=0$. This renders the last two terms in \cref{eq:rif} trivially zero when $\eta = \eta_P$. Their estimator is given exactly by removing from \cref{eq:rif} the last two terms, plugging in just $\hat h,\hat g$, and taking an empirical average. In our more general setting we can have $r_{P,\perp}\neq0$, or $a_{P,\perp}\neq0$, or neither, or both.
When $r_{P,\perp}=0$ is assumed, we may set $r_\perp\equiv0$ and omit $\alpha^h$;
when $a_{P,\perp}=0$ is assumed, we may set $a_\perp\equiv0$ and omit $\alpha^g$; when neither is assumed we need to learn all of $r_{P,\perp},a_{P,\perp},\alpha^h,\alpha^g$.
\end{rem}
When \(r_{P,\perp}=0\), the coordinate \(\alpha^h\) is inactive and is omitted
from the nuisance tuple; equivalently, all terms involving
\(r_{P,\perp}\alpha^h\) are interpreted as zero. When \(a_{P,\perp}=0\),
the coordinate \(\alpha^g\) is treated analogously. We let $\eta_P$ denote the population level (minimum-norm) nuisance tuple:
\begin{align}
\label{eq:eta_P_defn}
    \eta_P\doteq  (h_P^\dag,g_P^\dag,\alpha_P^{h,\dag},\alpha_P^{g,\dag},r_{P,\perp},a_{P,\perp}).
\end{align}
Our estimator will take the form of plugging in some $\hat\eta$ and taking the empirical expectation of $\chi$.

Toward understanding the behaviour of our estimator, we next characterize its bias.
\begin{thm}[Bias characterization]\label{thm: bias_char}
For any fixed $\eta=(h,g,\alpha^h,\alpha^g,r_\perp,a_\perp)$,
\begin{equation}
    \label{eq:simple_bias_expansion}
    \begin{aligned}
        \E_P[\chi(W;\eta)]-\Psi(P)
    &=
    -\left\langle
    T_P(h-h_P^\dag),g-g_P^\dag
    \right\rangle_{L^2(P_Z)}
\\
    &\quad+
    \left\langle
    T_P\alpha^h-g,r_\perp-r_{P,\perp}
    \right\rangle_{L^2(P_Z)}
\\
    &\quad+
    \left\langle
    T_P^*\alpha^g-h,a_\perp-a_{P,\perp}
    \right\rangle_{L^2(P_X)},
    \end{aligned}
\end{equation}
where the expectation on the left hand side is taken with respect to $W$. 
\end{thm}
\begin{proof}[Proof of \cref{thm: bias_char}]
By definition of the Riesz representers $a_P\in L^2(P_X)$ and $r_P\in L^2(P_Z)$, we have
\begin{align*}
    \E_P[m(W;h)] &= \langle a_P,h\rangle_{L^2(P_X)}\\
    \E_P[\widetilde m(W;g)] &= \langle r_P,g\rangle_{L^2(P_Z)}.
\end{align*}
We have
\begin{align*}
&\E_P[\chi(W;\eta)]\\ &= \langle a_P,h\rangle_{L^2(P_X)} +\langle r_P,g\rangle_{L^2(P_Z)} -\langle T_Ph,g\rangle_{L^2(P_Z)}  + \langle r_\perp,T_P\alpha^h-g\rangle_{L^2(P_Z)} +\langle a_\perp,T_P^*\alpha^g-h\rangle_{L^2(P_X)}.
\end{align*}
We substitute in the above equation
\begin{align*}
    a_P=T_P^*g_P^\dag+a_{P,\perp}, \qquad r_P=T_Ph_P^\dag+r_{P,\perp},
\end{align*}
and subtract $\Psi(P)=\langle T_Ph_P^\dag,g_P^\dag\rangle_{L^2(P_Z)}$, to find
\begin{align*}
    &\E_P[\chi(W;\eta)] - \Psi(P)\\ &= \left\langle T_P^{\ast}g_P^{\dag} + a_{P, \perp}, h\right\rangle_{L^2(P_X)} + \left\langle T_Ph_P^{\dag} + r_{P, \perp}, g \right\rangle_{L^2(P_Z)} - \langle T_Ph_P^\dag,g_P^\dag\rangle_{L^2(P_Z)} - \langle T_P h, g\rangle_{L^2(P_Z)}\\
    &+ \langle r_\perp,T_P\alpha^h-g\rangle_{L^2(P_Z)} +\langle a_\perp,T_P^*\alpha^g-h\rangle_{L^2(P_X)}\\
    &= -\langle T_P(h-h_P^\dag),g-g_P^\dag\rangle_{L^2(P_Z)}\\
    &+ \underbrace{\langle r_\perp,T_P\alpha^h-g\rangle_{L^2(P_Z)} + \langle r_{P,\perp},g\rangle_{L^2(P_Z)} +  \langle a_\perp,T_P^*\alpha^g-h\rangle_{L^2(P_X)} + \langle a_{P,\perp},h\rangle_{L^2(P_X)}}_{(\dag)} .
\end{align*}
Since $r_{P,\perp}\in\kernel(T_P^*)$ and $a_{P,\perp}\in\kernel(T_P)$, we have
\begin{align*}
    \langle T_P\alpha^h,r_{P,\perp}\rangle=0, \qquad \langle T_P^*\alpha^g,a_{P,\perp}\rangle=0.
\end{align*}
Therefore
\begin{align*}
    (\dag) = \langle T_P\alpha^h-g,r_\perp-r_{P,\perp}\rangle + \langle T_P^*\alpha^g-h,a_\perp-a_{P,\perp}\rangle,
\end{align*}
which concludes the proof.
\end{proof}

\begin{rem}
\label{rem:fbias_alpha_nuisance_discussion}
    We see from the right hand side of \cref{eq:simple_bias_expansion}, the $\alpha$-nuisances $\alpha^h$ and $\alpha^g$ only appear after projection as $T_P\alpha^h$ and $T_P^{\ast}\alpha^g$. As we will see in \Cref{sec:inference}, this means we do not need to learn $\alpha^h$ and $\alpha^g$ themselves well in a mean-squared sense. We only require consistency in strong metric and convergence rates in weak metric. This is important since even though $\alpha^h_P$ and $\alpha^g_P$ are solutions to potentially ill-posed inverse problems, we do not suffer from slow convergence rates due to ill-posedness. For $\alpha^h$, we can interpret this as saying we only need to learn a nuisance function that is close to $g$, while also lying in the range of $T_P$. Indeed, the key step in the proof of \Cref{eq:simple_bias_expansion} is that $r_{P, \perp}\in \kernel(T_P^{\ast})$, hence $\langle T_P \alpha^h, r_{P,\perp}\rangle_{L^2(P_Z)} = 0$. This motivates the assumption $g_P^{\dag}\in \range(T_P)$, which we show in \Cref{lem:Riesz_representer_equiv_assum} to be equivalent to the source condition $a_{P, \parallel}\in \range(T_P^{\ast}T_P)$. 
\end{rem}

\begin{rem}[Comparison to the setting where solutions to linear inverse problem(s) exist]
\label{rem:rates_weaken}
Following \citet{bennett2025inference,bennett2023sourceconditiondoublerobust}, where only the first term in the bias expansion appears,
we can bound the first term by the product of strong and weak metrics, where we can weaken \textit{either} the $\Hcal$-norm or the $\Gcal$-norm:
\begin{align*}
    &\left|\left\langle T_P(h-h_P^\dag),g-g_P^\dag\right\rangle_{L^2(P_Z)}\right|\\
&\leq
\min\left\{\|T_P(h-h_P^\dag)\|_{L^2(P_Z)}\,\|g-g_P^\dag\|_{L^2(P_Z)},
\|h-h_P^\dag\|_{L^2(P_X)}\,\|T_P^*(g-g_P^\dag)\|_{L^2(P_X)}
\right\}.
\end{align*}
In addition to this term, unlike \citet{bennett2025inference,bennett2023sourceconditiondoublerobust}, we also need to control two additional terms in \cref{eq:simple_bias_expansion}, due to the fact that we do not necessarily assume the orthogonal residuals $r_{P, \perp}$ and $a_{P,\perp}$ are zero.  
\end{rem}


\section{Nuisance learners}\label{sec:minimax_learners}

In this section we develop learners for all of our nuisance functions. For simplicity we will assume we have separate independent datasets $\mathcal{I}_1, \mathcal{I}_2, \mathcal{I}_3, \mathcal{I}_4$, each of size $n$. On $\mathcal{I}_1$, we will fit the primary and secondary nuisances $h_P^\dag$ and $g_P^\dag$ (\cref{sub:minimax_learners_of_h_p_dag}). Conditioning on $\mathcal{I}_1$, when needed we will fit $\alpha_P^{h,\dag}$ and $\alpha_P^{g,\dag}$ on $\mathcal{I}_2$ (\cref{sub:minimax_learners_of_the_weak_norm_riesz_representer}). Finally, conditioning on $\mathcal{I}_1$, we will fit $r_{P,\perp}$ and $a_{P,\perp}$ on $\mathcal{I}_3$ (\cref{sub:learners_of_orthogonal_residuals}). We will reserve $\mathcal{I}_4$ to construct a debiased estimator in \Cref{sec:inference}. Since all rates are expressed in terms of the sample size $n$ per dataset, using a fixed number of independent datasets does not affect our asymptotic results. Below, the sample of size $n$ used in each case will be implicit. To lighten notations, we do not denote which dataset we are referring to, as it is clear from context. All proofs in this section can be found in Appendix~\ref{sec:minimax_proofs}. 

\subsection{Background on empirical process minimization}
\label{sec:background}
Let $\mathcal{F}$ be a class of measurable functions $f: \mathcal{W}\to \mathbb{R}$. We define the star hull of $\mathcal{F}$ as 
\begin{align*}
    \mathrm{star}(\mathcal{F}) \doteq \{tf: f\in \mathcal{F}, t\in [0,1]\}. 
\end{align*}
We define the \emph{population localized Rademacher complexity} \citep{localradademacher} of the star hull of $\mathcal{F}$:
$$
\Rcal(\Fcal, \delta) \doteq \frac1{2^n}\sum_{\epsilon\in\{-1,1\}^n}\E_{P}\left[\sup_{\gamma\in[0,1],f\in \Fcal\,:\,\gamma \|f\|\leq \delta} \left|\frac{1}{n}\sum_{i=1}^n \epsilon_i \gamma f(W_i)\right|\right],$$
and define the \emph{critical radius} of $\mathcal{F}$ as the smallest positive solution to the inequality
\begin{align*}
    \mathcal{R}(\Fcal, \delta)\leq \delta^2.
\end{align*}
We refer the reader to \citet[Chapter 14]{wainwright2019high} for a detailed introduction to localized uniform concentration inequalities, and to \citet[Example 14.6, 14.7]{wainwright2019high} for examples of critical radii of kernel classes. In our work, we use a variant for vector-valued function classes and Lipschitz loss functions derived in \citet{foster2023orthogonal}. It is stated in \Cref{lem:lipschitz_wainwright} in the Appendix. 

\subsection{Minimax $h$-and $g$-learners}
\label{sub:minimax_learners_of_h_p_dag}
We construct learners of the minimum-norm least-squares solutions $h_P^\dag=\argmin_{h\in\Hcal_P}\|h\|_{L_2(P_X)}$ and $g_P^\dag=\argmin_{g\in\Gcal_P}\|g\|_{L_2(P_Z)}$, where $\Hcal_P$ and $\Gcal_P$ are defined in \cref{eq:hcal_P} and \cref{eq:gcal_P}. To do so, we show that the Tikhonov-regularized minimax learner of \citet{bennett2023sourceconditiondoublerobust} remains a good learner of $h_P^\dag$ even when no solution exists to $T_Ph=r_P$, so that $T_Ph_P^\dag\neq r_P$. Because $\Hcal_P$ and $\Gcal_P$ are defined symmetrically, whatever we develop also applies immediately for learning $g_P^\dag$, the minimum norm element of $\Gcal_P$.

Following \citet{bennett2023sourceconditiondoublerobust}, given a hypothesis class $\Hcal_n\subseteq\Hcal$ and critic class $\Gcal_n\subseteq\Gcal$, define the minimax learner
\begin{align}\label{eq:minimaxlearnerl}
    \widehat h(\lambda) \doteq \argmin_{h\in\Hcal_n}\max_{g\in\Gcal_n}
    \E_n\left[
    2\{\widetilde m(W;g)-h(X)g(Z)\}-g(Z)^2+\lambda h(X)^2
    \right].
\end{align}
The learner for $g_P^\dag$ is defined symmetrically, swapping $(\Hcal,X,T_P,\widetilde m,r_P)$ with $(\Gcal,Z,T_P^*,m,a_P)$.

For generality, in this section we impose a generic source condition. Note that, unlike \citet{bennett2023sourceconditiondoublerobust}, we impose it on the projection of $r_P$ onto $\overline{\range(T_P)}$ rather than on $r_P$ itself, since here we do not assume it lives in $\overline{\range(T_P)}$.
\begin{assum}[$\beta_h$-source condition for projected primary nuisance]\label{assum:betasource}For $\beta_h\geq0$, we have
\begin{align*}
    r_{P,\parallel}\in\range\left((T_PT_P^{\ast})^{(\beta_h+1)/2}\right).
\end{align*}
\end{assum}
For completeness, we also state the analogue of \Cref{assum:betasource} for $a_{P, \parallel}$, as we will make use of the source condition exponent in \Cref{sec:inference}.
\begin{assum}[$\beta_g$-source condition for projected secondary nuisance]\label{assum:betasource_g}For $\beta_g\geq0$, we have
\begin{align*}
    a_{P,\parallel}\in\range\left((T_P^{\ast}T_P)^{(\beta_g+1)/2}\right).
\end{align*}
\end{assum}
\begin{rem}[Relation of \Cref{assum:betasource} and \Cref{assum: source_condition_riesz}]
\label{rem:comparison_src}
    An alert reader will notice that \Cref{assum:betasource} with $\beta_h = 1$ is the same condition in \Cref{assum: source_condition_riesz}, if we \emph{need} to estimate an element of $\mathcal{A}_P^g$ for debiasing. Indeed, as explained after that assumption, we need to ensure that the Riesz regression solution set $\mathcal{A}_P^g$ is non-empty for debiasing when the residual $a_{P, \perp}$ is not assumed to be zero, and $r_{P, \parallel} \in \range(T_PT_P^{\ast})$ is precisely the \emph{equivalent} condition for $\mathcal{A}_P^{g}$ to be non-empty, as we show in \Cref{lem:Riesz_representer_equiv_assum}. In this respect, \Cref{assum: source_condition_riesz} is stronger than \Cref{assum:betasource}. Nevertheless, if we do \emph{not} require an element of $\mathcal{A}_P^{g}$ for debiasing, then we always impose a non-trivial source condition in the present assumption, whereas we do not restrict where $r_{P,\parallel}$ lies in \Cref{assum: source_condition_riesz}. Indeed, \Cref{assum:betasource} is equivalent to a source condition on the minimum norm primary nuisance $h_P^{\dag}$, 
    \begin{align*}
    h_P^{\dag} \in \range\left(\left(T_P^{\ast}T_P\right)^{\frac{\beta_h}{2}}\right),
\end{align*}
as we show in \Cref{cor: equiv_hpdag_source}. Such a source condition leads to faster convergence rates for the minimax $h$-learner. Indeed, source condition is the key condition to control the bias of a population-level solution to a ridge regularization problem, in this case $h_P^{\lambda}$ defined in \cref{eqn:populatoin_objective}. We will discuss further the role of the source condition in controlling bias in \Cref{rem:saturation}. 
\end{rem}
The above bias argument is shown in \Cref{lem:bias-tikhonov} as follows. 
\begin{lem}[Tikhonov regularization bias]\label{lem:bias-tikhonov}
Let $h_P^{\lambda}$ be as defined in \cref{eqn:populatoin_objective}. Suppose \cref{assum:betasource} holds. Then, for a finite constant $C_{P,\beta_h}$,
\begin{align*}
    \|h_P^{\lambda} - h_P^\dag\|_{L^2(P_X)}^2 \leq~& C_{P,\beta_h}\, \lambda^{\min\{\beta_h, 2\}}, \\
    \|T_P(h_P^{\lambda} - h_P^\dag)\|_{L^2(P_Z)}^2 \leq~& C_{P,\beta_h}\, \lambda^{\min\{\beta_h + 1, 2\}}.
\end{align*}
\end{lem}
Our proof of \Cref{lem:bias-tikhonov} adapts the proof \citet[Lemma 3]{bennett2023sourceconditiondoublerobust} to the case where $T_P h_P^\dag \neq r_P$, as well as to the setting where $T_P$ is not necessarily compact. To address the latter challenge, we apply the spectral theorem for bounded self-adjoint operators. We provide relevant mathematical background in \Cref{sec:minimax_proofs_bg} in the Appendix. 
\begin{rem}[Saturation effect]
\label{rem:saturation}
    A larger source condition exponent indicates that $h_P^{\dag}$ is smoother relative to the spectrum of $T_P^{\ast}T_P$, and therefore the bias will vanish at a faster rate. There's a catch to the story. Since we employ Tikhonov regularization, the convergence rates in the weak metric will improve up until $\beta_h = 1$, and in the strong metric will improve up until $\beta_h = 2$. This is reflected by the upper bound in \cref{lem:bias-tikhonov}, and is due to the well-known saturation effect of Tikhonov regularization \citep{engl1996regularization, carrasco2007linear}. In general, Tikhonov regularization falls within a broader family of regularization strategies known as spectral algorithms \citep{gerfo2008spectral}, and the saturation threshold is proportional to the \emph{qualification} of a spectral algorithm. As an example of spectral algorithm, we can instead use iterated Tikhonov regularization as is done in \citet{bennett2023sourceconditiondoublerobust}, whose qualification is proportional to the number of iteration steps. In the context of \Cref{rem:comparison_src}, under \Cref{assum: source_condition_riesz} we already require $\beta_h\geq 1$ when $a_{P, \perp}$ is not assumed to be zero, and we do \emph{not} obtain further convergence rates speed up in weak metric by taking $\beta_h > 1$, unless we modify our regularization strategy. 
\end{rem}
We next assume our hypothesis and critic classes are compatible in that the latter captures the directions of interest induced by the former. We also assume mean-squared-continuity.
\begin{assum}[Closedness]\label{assum:wellspecified}
$T_P(h_P^\dag-\Hcal_n)\subseteq\Gcal_n$.
\end{assum}

\begin{assum}[Mean-squared 
continuity]\label{assum:mscont}
There exists a universal constant $C>0$, such that for all $g\in \Gcal$, $h\in \Hcal$, we have
\begin{align*}
    \E_P\left[\widetilde m(W;g)^2\right] \leq C\|g\|^2_{L^2(P_Z)}, \quad \E_P\left[m(W;h)^2\right] \leq C\|h\|^2_{L^2(P_X)}. 
\end{align*}
\end{assum}

We are now prepared to state the equivalent of \citet[Theorem 4]{bennett2023sourceconditiondoublerobust} but extended to the case where a solution to the linear equation of interest need not exist:
\begin{thm}[Convergence rates for $h$-learners]
\label{thm:minimaxlearner}
Fix $\lambda\le1$. Suppose
\cref{assum:betasource,assum:wellspecified,assum:mscont} hold, there exists a constant $C>0$ independent of $n$ such that 
$\widetilde m(W;g)$, $g(Z)$, and $h(X)$ are almost surely bounded by
$C$ uniformly over $g\in\Gcal_n,h\in\Hcal_n$, and
$h_P^\dag,h_P^\lambda\in\Hcal_n$. Let $\zeta\in(0,0.5)$ and let
\[
    \delta_{h,n}
    =
    \Omega\left(\sqrt{\frac{\log\log n-\log\zeta}{n}}\right)
\]
be an upper bound on the critical radii of
$\Hcal_n\cdot\Gcal_n$, $\{\widetilde m(\cdot;g):g\in\Gcal_n\}$,
$\Hcal_n$, and $\Gcal_n$. Then, with probability at least $1-\zeta$,
\[
\|\widehat h(\lambda)-h_P^\dag\|_{L^2(P_X)}^2
=
O\left(\frac{\delta_{h,n}^2}{\lambda}+\lambda^{\min\{\beta_h,1\}}\right),
\]
and
\[
\|T_P(\widehat h(\lambda)-h_P^\dag)\|_{L^2(P_Z)}^2
=
O\left(\delta_{h,n}^2+\lambda^{\min\{\beta_h+1,2\}}\right).
\]
In particular, with
$\lambda_h=\delta_{h,n}^{2/\min\{\beta_h+1,2\}}$,
\[
\|\widehat h(\lambda_h)-h_P^\dag\|_{L^2(P_X)}^2
=
O\left(\delta_{h,n}^{2\min\{\beta_h,1\}/\min\{\beta_h+1,2\}}\right),
\qquad
\|T_P(\widehat h(\lambda_h)-h_P^\dag)\|_{L^2(P_Z)}^2
=
O(\delta_{h,n}^2).
\]
\end{thm}

\subsection{Minimax $\alpha$-learners}
\label{sub:minimax_learners_of_the_weak_norm_riesz_representer}
From \Cref{rem:bennett_exact_solution}, we know that if $r_{P,\perp}\neq 0$, we need to learn an element of $\mathcal{A}_P^{h}$ for debiasing; symmetrically, if $a_{P,\perp}\neq 0$, we need to learn an element of $\mathcal{A}_P^{g}$ for debiasing, where $\Acal_P^h$ and $\Acal_P^g$ are defined in \Cref{eq: alpha_h_P_riesz_regression} and \Cref{eq: alpha_g_P_riesz_regression}. We only consider the former learning problem, all of our results apply symmetrically to the latter. We assume $a_{P,\parallel} \in \mathrm{range}(T_P^{\ast}T_P)$ so that $\mathcal{A}_P^{h}\neq \emptyset$, by \Cref{lem:Riesz_representer_equiv_assum}. Under \Cref{assum: source_condition_riesz}, this condition is required whenever $r_{P,\perp}\neq0$. We target the minimum norm element of $\mathcal{A}_P^{h}$
\begin{equation*}
    \alpha_P^{h,\dag} \doteq \argmin_{\alpha^h\in\Acal_P^h}\|\alpha^h\|_{L^2(P_X)}.
\end{equation*}

We define the following minimax learner for $\alpha_P^{h,\dag}$. 
Given a hypothesis class $\mathcal{H}_n \subseteq\mathcal{H}$, critic class $\mathcal{G}_n \subseteq \mathcal{G}$, and Tikhonov regularization parameter $\lambda_{\alpha^h}>0$, we define the Tikhonov regularized minimax learner:
\begin{align}
    \widehat\alpha^h(\lambda_{\alpha^h})
    &\doteq
    \argmin_{\alpha^h\in\Hcal_n}\max_{g\in\Gcal_n}
    \E_n\left[
    2\{\widehat{g}(\lambda_g)(Z)g(Z)-\alpha^h(X)g(Z)\}-g(Z)^2+\lambda_{\alpha^h}\alpha^h(X)^2
    \right],
    \label{eq: alpha_h_feasible_minimax}
\end{align}
where $\widehat g(\lambda_g)$ is a minimax learner for $g_P^{\dag}$ as defined in \Cref{sub:minimax_learners_of_h_p_dag}, fitted on an independent split of the data from the empirical average in \cref{eq: alpha_h_feasible_minimax}.

\begin{lem}[Bias bounds]
\label{lem:ppl_reg_bias_alpha}
We assume $a_{P,\parallel} \in \mathrm{range}(T_P^{\ast}T_P)$. For $\lambda>0$, we define
\begin{align*}
    \alpha_{P,\lambda}^h \doteq \argmin_{\alpha\in\Hcal}
    \left\{\|T_P\alpha-g_P^\dag\|_{L^2(P_Z)}^2+\lambda\|\alpha\|_{L^2(P_X)}^2\right\}. 
\end{align*}
We now take $\lambda = \lambda_n \downarrow 0$ to be a monotonically decreasing sequence of positive real numbers. We have, as $n\to\infty$, 
\begin{align*}
    \|T_P(\alpha_{P,\lambda_n}^h-\alpha_P^{h,\dag})\|_{L^2(P_Z)}^2 &= o(\lambda_n) \\
    \|\alpha_{P,\lambda_n}^h-\alpha_P^{h,\dag}\|_{L^2(P_X)}^2 &= o(1). 
\end{align*}
\end{lem}
The proof of \Cref{lem:ppl_reg_bias_alpha} leverages the spectral theorem for bounded self-adjoint operators \citep{reed1972methods}, which allows us to avoid assuming $T_P$ is compact. We compare \Cref{lem:ppl_reg_bias_alpha} to \Cref{lem:bias-tikhonov} in the case where we don't assume a non-trivial source condition. The reason we're able to obtain $o(\lambda)$ and $o(1)$ instead of $O(\lambda)$ and $O(1)$ is because we apply dominated convergence theorem with respect to a fixed target $\alpha_P^{h,\dag}$.

\begin{assum}[Closedness]\label{assum:wellspecified_weakriesz}
$T_P(\alpha_P^{h,\dag}-\mathcal{H}_n)\subseteq\mathcal G_n$.
\end{assum}

\begin{thm}[Convergence rates for $\alpha$-learners]
\label{thm:minimaxlearner_weak_riesz_feasible}
Let $\lambda_{\alpha^h,n}\le 1$. Suppose $a_{P,\parallel}\in\range(T_P^*T_P)$, \cref{assum:wellspecified_weakriesz} holds, and $0\in\Gcal_n$. Suppose also that
\[
    \alpha_P^{h,\dag},\alpha_{P,\lambda_{\alpha^h,n}}^h\in\Hcal_n .
\]
Assume that there exists a constant $C>0$ independent of $n$, such that $g(Z)$ and $\alpha(X)$ are almost surely bounded by $C$ uniformly over $g\in\Gcal_n$ and $\alpha\in\Hcal_n$. Let $\widehat g(\lambda_g)$ be trained on a sample independent of the sample used in the sample average in \cref{eq: alpha_h_feasible_minimax}, and conditionally on the first-stage sample suppose $\widehat g(\lambda_g)(Z)$ is almost surely bounded by $O(1)$. Let $\zeta\in(0,0.5)$, and let
\[
    \delta_{\alpha^h,n}
    =
    \Omega\left(
        \sqrt{\frac{\log\log n+\log(1/\zeta)}{n}}
    \right)
\]
be a common upper bound on the critical radii of
\[
    \Gcal_n,\qquad
    \Hcal_n\cdot\Gcal_n,\qquad
    \Hcal_n-\alpha_{P,\lambda_{\alpha^h,n}}^h .
\]
Then, conditionally on the first-stage sample, with probability at least $1-\zeta$,
\[
    \|T_P(\widehat\alpha^h-\alpha_P^{h,\dag})\|_{L^2(P_Z)}^2
    \le
    C\left(
        \delta_{\alpha^h,n}^2
        +
        \lambda_{\alpha^h,n}
        +
        \|\widehat g(\lambda_g)-g_P^\dag\|_{L^2(P_Z)}^2
    \right),
\]
for a constant $C<\infty$ independent of $n$ and $\zeta$. If $\lambda_{\alpha^h,n}\downarrow0$ and
\[
    \frac{
        \delta_{\alpha^h,n}^2
        +
        \|\widehat g(\lambda_g)-g_P^\dag\|_{L^2(P_Z)}^2
    }{\lambda_{\alpha^h,n}}
    \to_p0,
\]
then
\[
    \|\widehat\alpha^h-\alpha_P^{h,\dag}\|_{L^2(P_X)}^2\to_p0.
\]
\end{thm}
Our proof relies on showing that 
\begin{align*}
    \lambda_{\alpha^h,n}\|\widehat\alpha^{h}-\alpha_P^{h,\dag}\|_{L^2(P_X)}^2 +  \|T_P(\widehat\alpha^{h}-\alpha_P^{h,\dag})\|_{L^2(P_Z)}^2
\end{align*}
can be bounded using $\delta_{\alpha^h,n}$ and the bias terms given in \Cref{lem:ppl_reg_bias_alpha}. This step uses a localized uniform concentration inequality stated in \Cref{lem:lipschitz_wainwright}. As we do not impose a non-trivial source condition on $\alpha_P^{h,\dag}$, \Cref{thm:minimaxlearner_weak_riesz_feasible} gives weak metric convergence rates but only consistency with respect to the strong metric.

\subsection{$r$-learners}
\label{sub:learners_of_orthogonal_residuals}

Finally, we learn the orthogonal residual
\begin{align*}
    r_{P,\perp} = r_P - T_P h_P^{\dag} = r_P  - r_{P, \parallel}.
\end{align*}
We focus on $r_{P,\perp}$; the learner for $a_{P,\perp}$ follows by exchanging notations
\begin{align*}
    \left(T_P,\Gcal,Z,\widetilde m,r_P,h_P^\dag\right) \leftrightarrow \left(T_P^\ast,\Hcal,X,m,a_P,g_P^\dag\right).
\end{align*}
For a given first-stage learner $h_1$, our learning target is given by $r_P-T_Ph_1$. Given a fixed $h_1\in\Hcal$ and a function class $\Gcal_{n}^{r}\subseteq\Gcal$, for $\widetilde{m}$ given in \Cref{eq:m_tilde}, we define
\begin{align}
    \widehat r_\perp(h_1) \in \argmax_{q\in\Gcal_{n}^{r}}
    \E_n\left[2\{\widetilde m(W;q)-h_1(X)q(Z)\}-q(Z)^2\right].\label{eq:rperp_learner}
\end{align}
To see why \Cref{eq:rperp_learner} is the correct objective, we observe that
\begin{align}
\label{eq:widehatrperpid}
    \mathbb{E}\left[2\widetilde{m}(W;q) - 2h_1(X)q(Z) - q(Z)^2\right] = \mathbb{E}\left[2(r_P(Z) - (T_Ph_1)(Z))q(Z) - q(Z)^2\right].
\end{align}
Differentiating the right hand side with respect to $q$ formally yields $q = r_P - T_P h_1$.

\begin{thm}[$L^2$-norm error of $\widehat r_\perp(h_1)$ with fixed first stage nuisance $h_1$]
\label{thm:orthogonal_residual_fixed}
Fix $h_1\in\Hcal$. We make the following \emph{realizability} assumption
\begin{align*}
    r_P-T_Ph_1\in\Gcal_{n}^{r},
\end{align*}
suppose that \Cref{assum:mscont} holds, and that there exists a constant $C>0$ independent of $n$, such that $\widetilde m(W;q)$, $q(Z)$, $h_1(X)$, and
$(r_P-T_Ph_1)(Z)$ are almost surely bounded by $C$, uniformly over
$q\in\Gcal_{n}^{r}$. Let $\zeta\in(0,0.5)$ and let
\[
    \delta_{r_\perp,n} = \Omega\left(\sqrt{\frac{\log\log n-\log\zeta}{n}}\right)
\]
be an upper bound on the critical radii of $\Gcal_{n}^{r}$ and
$\{\widetilde m(\cdot;q):q\in\Gcal_{n}^{r}\}$. Then, with probability at least
$1-\zeta$,
\[
    \big\|\widehat r_\perp(h_1)-(r_P-T_Ph_1)\big\|_{L^2(P_Z)}^2 = O(\delta_{r_\perp,n}^2).
\]
\end{thm}
The motivation for the realizability assumption is clear from \Cref{eq:widehatrperpid}.

\begin{cor}[Convergence rates for $r$-learners]
\label{cor:rperp_plugin}
Suppose $\widehat h(\lambda_h)$ is trained on a sample independent of the sample used in the sample average in \Cref{eq:rperp_learner}. We also suppose that the assumptions of \Cref{thm:minimaxlearner} hold for
$\widehat h(\lambda_h)$, and the assumptions of
\Cref{thm:orthogonal_residual_fixed} hold with
$h_1=\widehat h(\lambda_h)$. Then
\[
    \big\|\widehat r_\perp(\widehat h(\lambda_h))-r_{P,\perp}\big\|_{L^2(P_Z)}^2 = O_p\left(
    \delta_{r_\perp,n}^2+\delta_{h,n}^2+
    \lambda_h^{\min\{\beta_h+1,2\}}
    \right).
\]
In particular, if $\lambda_h=\delta_{h,n}^{2/\min\{\beta_h+1,2\}}$, then with probability at least $1-\zeta$,
\[
    \big\|\widehat r_\perp(\widehat h(\lambda_h))-r_{P,\perp}\big\|_{L^2(P_Z)}^2 = O(\delta_{r_\perp,n}^2+\delta_{h,n}^2).
\]
\end{cor}

\begin{rem}[Saddle point vs second max step]
\Cref{eq:rperp_learner} is exactly the inner max of the minimax learner \Cref{eq:minimaxlearnerl} for $\widehat{h}(\lambda)$, so taking both solutions to that minimax problem at optimality can potentially give a simultaneous estimator for both $h_P^\dag$ and $r_{P,\perp}$.  For simplicity of analysis alone, in this section we considered instead a two-stage procedure, where we optimize the inner max on a new sample, with $h_1$ set to $\widehat{h}(\lambda_h)$. In practice, we can and should use a single minimax learner for the saddle point. Nonetheless, it is important to note that additional realizability would be needed: \cref{assum:wellspecified} only requires elements of $\range(T_P)$ to lie in the critic class $\mathcal{G}_n$, while \cref{cor:rperp_plugin} may require $\mathcal{G}_n^r$ to contain elements that also have a projection onto $\kernel(T_P)$. Therefore, a simultaneous minimax estimator for the saddle point, having a single critic class, would have to satisfy both.
\end{rem}


\section{Debiased inference on the target functional}\label{sec:inference}

We now present our main results on the inference for $\Psi(P)$, building on the minimax nuisance learners and non-asymptotic bounds from the preceding section, and the functional bias expansion in \Cref{thm: bias_char}.

We view the set of nuisance functions as a \emph{tuple} and refer to a particular nuisance function as a \emph{coordinate} of the tuple. As explained at the start of \Cref{sec:minimax_learners}, we use three independent datasets $\mathcal{I}_1, \mathcal{I}_2, \mathcal{I}_3$ to learn the nuisances, and we use a fourth independent dataset $\mathcal{I}_4$ to construct the debiased estimator, defined as follows, 
\begin{align}
    \widehat{\Psi} \doteq \frac{1}{n}\sum_{i\in \mathcal{I}_4} \chi(W_i;\widehat{\eta}). 
    \label{eq:autodml-estimator}
\end{align}
where the function $\chi$ acting on data sample and nuisance tuple is defined in \cref{eq:rif}, $\widehat\eta=(\widehat h,\widehat g,\widehat\alpha^h,\widehat\alpha^g,\widehat r_\perp,\widehat a_\perp)$ is fitted on $\mathcal{I}_1, \mathcal{I}_2, \mathcal{I}_3$ as described in \Cref{sec:minimax_learners}. We also define an estimator for the asymptotic variance of $\widehat{\Psi}$ as 
\begin{align}
\label{eq:var_estimator}
    \widehat{\sigma}^2 \doteq \frac{1}{n}\sum_{i\in \mathcal{I}_4} (\chi(W_i; \hat{\eta})-\widehat{\Psi})^2, 
\end{align}
The extension to cross-fitting \citep{Chernozhukov2018,chernozhukov2025applied} is immediate: we split the data and permute the roles, averaging the estimator $\widehat{\Psi}$ over these, each time with $\hat\eta$ fit on disjoint folds. Then we can retain the $\sqrt n$-asymptotic-linearity we will obtain in this section without need for additional samples. To keep our analysis simple, we essentially just consider a single split. 

As in \Cref{sec: functional_bias}, we use the following tuple as the fixed target of our nuisance learners, 
\[
    \eta_P^\dag \doteq  (h_P^\dag,g_P^\dag,\alpha_P^{h,\dag},\alpha_P^{g,\dag},
    r_{P,\perp},a_{P,\perp}),
\]
where the first four coordinates are minimum-norm elements of their respective sets. We use the convention that the $\alpha^h$ coordinate is omitted when $r_{P,\perp}=0$, and the $\alpha^g$ coordinate is omitted when
$a_{P,\perp}=0$. If a coordinate is omitted, we refer to it as an \emph{"inactive"} coordinate. Otherwise, we refer to it as an \emph{"active"} coordinate. The function $w\mapsto \chi(w;\eta_P^{\dag})$ is an (uncentered) \emph{influence function candidate}. By \Cref{thm: bias_char}, $\E_P[\chi(W; \eta_P^{\dag})]=\Psi(P)$. We now give names to this mapping and its centred version, 
\begin{align}
    \label{eq:chi_P_defn}
    \chi_P(w)\doteq \chi(w;\eta_P^\dag),
    \qquad
    \varphi_P(w)\doteq \chi_P(w)-\Psi(P),
    \qquad
    \sigma_P^2\doteq \E_P[\varphi_P(W)^2].
\end{align}
We emphasize that despite parametrization invariance shown in \Cref{lem: param_inv}, the candidate $\sigma_P^2$ for our asymptotic variance is defined in terms of \emph{minimum-norm} elements of the least squares solution sets. In general, choosing arbitrary elements of the least squares solution sets may lead to a larger asymptotic variance candidate. 

We define the strong and weak metric errors for the nuisance functions
\begin{align*}
e_h^s&=\|\widehat h-h_P^\dag\|_{L^2(P_X)},
&
e_h^w&=\|T_P(\widehat h-h_P^\dag)\|_{L^2(P_Z)},\\
e_g^s&=\|\widehat g-g_P^\dag\|_{L^2(P_Z)},
&
e_g^w&=\|T_P^*(\widehat g-g_P^\dag)\|_{L^2(P_X)},\\
e_{\alpha^h}^s&=\|\widehat\alpha^h-\alpha_P^{h,\dag}\|_{L^2(P_X)},
&
e_{\alpha^h}^w&=\|T_P(\widehat\alpha^h-\alpha_P^{h,\dag})\|_{L^2(P_Z)},\\
e_{\alpha^g}^s&=\|\widehat\alpha^g-\alpha_P^{g,\dag}\|_{L^2(P_Z)},
&
e_{\alpha^g}^w&=\|T_P^*(\widehat\alpha^g-\alpha_P^{g,\dag})\|_{L^2(P_X)},\\
e_{r_\perp}^s&=\|\widehat r_\perp-r_{P,\perp}\|_{L^2(P_Z)},
&\\
e_{a_\perp}^s&=\|\widehat a_\perp-a_{P,\perp}\|_{L^2(P_X)}.
\end{align*}
We only require consistency for $e_{\alpha^h}^s$ and $e_{\alpha^g}^s$. We require convergence rates for all the other errors. These results are provided by the non-asymptotic guarantees in \Cref{sec:minimax_learners}. If we assume that $r_{P,\perp}=0$ and we set $\widehat r_\perp\equiv0$, take
$e_{\alpha^h}^s=e_{\alpha^h}^w=e_{r_\perp}^s=0$. If we assume that
$a_{P,\perp}=0$ and we set $\widehat a_\perp\equiv0$, take
$e_{\alpha^g}^s=e_{\alpha^g}^w=e_{a_\perp}^s=0$. 

\begin{cond}[Requirements on convergence rates for nuisance functions]\label{cond:inference_nuisance}
The nuisance errors are consistent with respect to the strong metric,
\begin{align}
\label{eq:consistency}
    e_h^s+e_g^s
+e_{\alpha^h}^s
+e_{\alpha^g}^s
+e_{r_\perp}^s+e_{a_\perp}^s
=o_p(1)
\end{align}
Since $T_P$ and $T_P^*$ are bounded, this also implies weak metric consistency
for the corresponding coordinates. We also require
\begin{align}
\sqrt n\min\{e_g^w e_h^s,\ e_g^s e_h^w\}&=o_p(1), \label{eq:product_err_1}\\
\sqrt n\left[
(e_{\alpha^h}^w+e_g^s)e_{r_\perp}^s
+(e_{\alpha^g}^w+e_h^s)e_{a_\perp}^s
\right]&=o_p(1). \label{eq:product_err_2}
\end{align}
\end{cond}
We refer to \cref{eq:product_err_1} and \cref{eq:product_err_2} as \emph{product rate conditions}. Note \cref{eq:product_err_1} is derived from \Cref{rem:rates_weaken}. We also observe that \cref{eq:product_err_2} contains plug-in $L^2$ errors $e_h^s$ and $e_g^s$ for the primary and secondary nuisances respectively, while only weak metric errors $e_{\alpha^h}^w$ and $e_{\alpha^g}^w$ for the $\alpha$-nuisances. We refer the reader to our discussion in  \Cref{rem:fbias_alpha_nuisance_discussion} earlier in the paper.

\begin{thm}[Asymptotic linearity]\label{thm:autodml-clt}
Suppose \Cref{assum: range_condition,assum: source_condition_riesz,assum:mscont} and
\Cref{cond:inference_nuisance} hold, and $0<\sigma_P^2<\infty$. We also assume that there exists a large constant $C>0$ such that
\begin{align*}
    \max\left\{\left\|\widehat{h}\right\|_{\infty}^2,\left\|h^{\dag}_P\right\|^2_{\infty}, \left\|g^{\dag}_P\right\|^2_{\infty}, \left\|\alpha_P^{h,\dag}\right\|_{\infty}^2, \left\|\alpha_P^{g,\dag}\right\|_{\infty}^2, \left\|\widehat{a}_{\perp}\right\|^2_{\infty},\left\|\widehat{r}_{\perp}\right\|^2_{\infty} \right\} \leq C,
\end{align*}
with probability tending to $1$. We drop the assumptions on $\alpha_P^{h,\dag}$ and $\alpha_P^{g,\dag}$ if they are inactive. Then
\begin{align}
    \sqrt n(\widehat\Psi-\Psi(P))
    =
    \frac1{\sqrt n}\sum_{i=1}^n\varphi_P(W_i)+o_p(1)
    \rightsquigarrow N(0,\sigma_P^2).
    \label{eq:asymp_linear}
\end{align}
Moreover, $\widehat{\sigma}$ is a consistent estimator for the asymptotic standard deviation $\sigma_P$, i.e. $\widehat\sigma\to_p\sigma_P$, where $\widehat{\sigma}$ is defined in \cref{eq:var_estimator}.
\end{thm}

\subsection{Sufficient conditions for nuisance learners}
We next record sufficient conditions under which the nuisance learners from \Cref{sec:minimax_learners} satisfy \Cref{cond:inference_nuisance}. Let $\beta_h$ and $\beta_g$ be defined as in \Cref{assum:betasource,assum:betasource_g}, and define
\[
\kappa_h\doteq\frac{\min\{\beta_h,1\}}{\min\{\beta_h+1,2\}},
\qquad
\kappa_g\doteq\frac{\min\{\beta_g,1\}}{\min\{\beta_g+1,2\}}.
\]
Recall that under \Cref{assum: source_condition_riesz}, $\beta_h \geq 1$ if $a_{P,\perp}\neq 0$, and $\beta_g\geq 1$ if $r_{P,\perp}\neq 0$. Let $\delta_{h,n}$, $\delta_{g,n}$, $\delta_{\alpha^h,n}$, $\delta_{\alpha^g,n}$, $\delta_{r_\perp,n}$, and $\delta_{a_\perp,n}$ denote the critical radii appearing in the nuisance-learning results for $\widehat h$, $\widehat g$, $\widehat\alpha^h$, $\widehat\alpha^g$, $\widehat r_\perp$, and $\widehat a_\perp$, respectively, as defined in \cref{thm:minimaxlearner,thm:minimaxlearner_weak_riesz_feasible,thm:orthogonal_residual_fixed}.

\Cref{thm:minimaxlearner} and its symmetric analogue give
\[
    e_h^w=O_p(\delta_{h,n}),\qquad
    e_h^s=O_p(\delta_{h,n}^{\kappa_h}),\qquad
    e_g^w=O_p(\delta_{g,n}),\qquad
    e_g^s=O_p(\delta_{g,n}^{\kappa_g}).
\]
For "active" coordinates of the $\alpha$-nuisances (i.e. $\alpha^h$ if $r_{P,\perp}\neq 0$ and $\alpha^g$ if $a_{P,\perp}\neq 0$), \Cref{thm:minimaxlearner_weak_riesz_feasible} and its symmetric analogue give
\begin{align*}
    e_{\alpha^h}^w
    &=
    O_p\left(
    \left\{\delta_{\alpha^h,n}^2+\lambda_{\alpha^h,n}
    +\delta_{g,n}^{2\kappa_g}\right\}^{1/2}
    \right),\\
    e_{\alpha^g}^w
    &=
    O_p\left(
    \left\{\delta_{\alpha^g,n}^2+\lambda_{\alpha^g,n}
    +\delta_{h,n}^{2\kappa_h}\right\}^{1/2}
    \right).
\end{align*}
Moreover, $e_{\alpha^h}^s=o_p(1)$ if
\begin{align}
    \label{eq:suff_cond_lambda_alpha_h}
    \lambda_{\alpha^h,n}\downarrow0,
    \qquad
    \frac{\delta_{\alpha^h,n}^2+\delta_{g,n}^{2\kappa_g}}
    {\lambda_{\alpha^h,n}}\to0,
\end{align}
and, symmetrically, $e_{\alpha^g}^s=o_p(1)$ if
\begin{align}
    \label{eq:suff_cond_lambda_alpha_g}
    \lambda_{\alpha^g,n}\downarrow0,
    \qquad
    \frac{\delta_{\alpha^g,n}^2+\delta_{h,n}^{2\kappa_h}}
    {\lambda_{\alpha^g,n}}\to0.
\end{align}
Note that \cref{eq:suff_cond_lambda_alpha_h} and \cref{eq:suff_cond_lambda_alpha_g} are exactly bias-variance trade-off conditions. The $\lambda$'s have to decay sufficiently quickly for bias of the estimators to be small, while they can't decay too fast to control the estimators'
variance. Finally, \Cref{cor:rperp_plugin} and its symmetric analogue give
\[
    e_{r_\perp}^s
    =
    O_p(\delta_{r_\perp,n}+\delta_{h,n}),
    \qquad
    e_{a_\perp}^s
    =
    O_p(\delta_{a_\perp,n}+\delta_{g,n}).
\]
Thus \cref{eq:consistency} follows if \cref{eq:suff_cond_lambda_alpha_h} and \cref{eq:suff_cond_lambda_alpha_g} hold, and $\delta_{h,n}, \delta_{g,n}, \delta_{r_{\perp}, n}$ and $\delta_{a_{\perp},n}$ are $o(1)$. The first product condition \cref{eq:product_err_1} is implied by
\[
    \sqrt n\min\left\{
    \delta_{g,n}\delta_{h,n}^{\kappa_h},~
    \delta_{g,n}^{\kappa_g}\delta_{h,n}
    \right\}
    \to0.
\]
For the second product condition \cref{eq:product_err_2}, define
\begin{align*}
    q^2_{\alpha^h,n}
    &\doteq
    \begin{cases}
    \delta_{\alpha^h,n}^2+\lambda_{\alpha^h,n}
    +\delta_{g,n}^{2\kappa_g}
    , & r_{P,\perp}\neq0,\\
    0, & r_{P,\perp}=0,
    \end{cases}\\
    q^2_{\alpha^g,n}
    &\doteq
    \begin{cases}
    \delta_{\alpha^g,n}^2+\lambda_{\alpha^g,n}
    +\delta_{h,n}^{2\kappa_h}
    , & a_{P,\perp}\neq0,\\
    0, & a_{P,\perp}=0.
    \end{cases}
\end{align*}
Then \cref{eq:product_err_2} is implied by
\[
    \sqrt n\left[
    q_{\alpha^h,n}(\delta_{r_\perp,n}+\delta_{h,n})
    +
    q_{\alpha^g,n}(\delta_{a_\perp,n}+\delta_{g,n})
    \right]\to0.
\]

\paragraph{Simple sufficient condition.}
Consider the case where we do not a priori assume either residual $r_{P, \perp}, a_{P, \perp}$ is zero and where a single critical radius $\delta_n$ bounds all of
$\delta_{h,n}$, $\delta_{g,n}$, $\delta_{\alpha^h,n}$, $\delta_{\alpha^g,n}$,
$\delta_{r_\perp,n}$, and $\delta_{a_\perp,n}$.
Define 
$\underline\kappa\doteq\min\{\kappa_h,\kappa_g\}$.
If
\[
    \sqrt n\,\delta_n^{1+\underline\kappa}\to0,
\]
then we have that \Cref{cond:inference_nuisance} holds  by choosing $\lambda_{\alpha^h,n}=\lambda_{\alpha^g,n}=\lambda_n=o(1)$ with $\lambda_n=\omega(\delta_n^{2\underline\kappa})$, $\lambda_{n}=o(n^{-1}\delta_n^{-2})$. If $\beta_h,\beta_g\ge1$, then
$\kappa_h=\kappa_g=1/2$ due to the saturation effect, so this amounts to requiring $\delta_n=o(n^{-1/3})$. 


\section{Conclusions}

In low-dimensional linear IV analysis, two-stage least squares (2SLS) does not require there to be an exact solution $\beta$ to $\E[Z X^\top]\beta=\E[ZY]$ \citep{theil1953repeated,basmann1957generalized,sargan1958estimation,anderson2005origins}. When there is no exact solution, it just targets the least-squares solution $\E[Z X^\top]^\dag\E[ZY]$ to this (population) equation \citep{andrews2019structure,andrews2025purpose}, while inference requires adjustment to covariance estimates \citep{hall2003large,maasoumi1982behavior}. 
When moving to NPIV, we have to estimate nonparametric functions to solve infinite-dimensional linear equations and to debias functionals to ensure valid inference on finite-dimensional parameters \citep{bennett2025inference,bennett2023sourceconditiondoublerobust,ai2003efficient,severini2012efficiency}. Along the way, much of the NPIV-functionals literature has imposed existence of solutions, whether to $T_Ph=r_P$ or $T_P^*g=a_P$, that were never needed for low-dimensional 2SLS. In this paper, we sought to remove these requirements in order to harmonize modern debiased inference on functionals of NPIV with the classic least-squares approach to low-dimensional linear IV. Doing so requires additional debiasing terms to account for nonzero residuals in the equations to be solved, which is the semiparametric analogue to the covariance corrections needed in finite-dimensional methods of moments \citep{hall2003large,maasoumi1982behavior}, except here this also changes the debiased estimator itself.

\section*{Acknowledgements}
The authors would like to thank Ben Deaner for helpful discussions. 

\bibliographystyle{plainnat}
\bibliography{references}   

\newpage
\appendix

\section{Proofs for functional bias}\label{sec:functional_bias_proofs}

\begin{proof}[Proof of \cref{lem:Riesz_representer_equiv_assum}]
We prove the statement for $\alpha^h$; the proof for $\alpha^g$ is identical
with $T_P$ replaced by $T_P^*$. The equivalence between
$a_{P,\parallel}\in\range(T_P^*T_P)$ and
$\mathcal A_P^h\neq\emptyset$ follows from the first-order condition. Indeed,
the objective in
Eq.~\eqref{eq: alpha_h_P_riesz_regression} is convex and Fréchet differentiable, so
$\alpha\in\Hcal$ minimizes it if and only if
\begin{align*}
    \langle T_P^*T_P\alpha-a_{P,\parallel},v\rangle_{L^2(P_X)}=0, \qquad \forall v\in\Hcal.
\end{align*}
Thus $\mathcal A_P^h$ is nonempty if and only if
$a_{P,\parallel}\in\range(T_P^*T_P)$, and any solution satisfies
\begin{align*}
    T_P^*T_P\alpha_P^h=a_{P,\parallel},
\end{align*}
Since $T_P^*g_P^\dag=a_{P,\parallel}$, we have
\begin{align*}
    T_P^*(T_P\alpha_P^h-g_P^\dag)=0.
\end{align*}
Also $T_P\alpha_P^h\in\range(T_P)$ and
$g_P^\dag\in\kernel(T_P^*)^\perp=\overline{\range(T_P)}$. Hence
\begin{align*}
    T_P\alpha_P^h-g_P^\dag \in \kernel(T_P^*) \cap \kernel(T_P^*)^\perp  = \{0\}.
\end{align*}
This proves that either of the first two statements implies
$g_P^\dag\in\range(T_P)$. Conversely, if
$g_P^\dag=T_P\alpha$ for some $\alpha\in\Hcal$, then
\begin{align*}
    a_{P,\parallel}=T_P^*g_P^\dag=T_P^*T_P\alpha,
\end{align*}
so $a_{P,\parallel}\in\range(T_P^*T_P)$. This proves the three
equivalences for $\alpha^h$, and the displayed identities above prove the
remaining claim for every $\alpha_P^h\in\mathcal A_P^h$.
\end{proof}


\section[Proofs for nuisance learners]{Proofs for \cref{sec:minimax_learners}}\label{sec:minimax_proofs}
\subsection{Mathematical background}
\label{sec:minimax_proofs_bg}
We follow \citet[Chapter VII]{reed1972methods} to present the following generalization of spectral theorem for self-adjoint compact bounded operators. It will enable us to analyse the setting with non-compact $T:\mathcal{H}\to \mathcal{G}$, which arises in e.g. proximal causal inference with continuous treatment \citep{deaner2018proxy} and NPIV with observed covariates \citep{horowitz2011applied, shen2025nonparametricinstrumentalvariableregression}.
\begin{defn}[Projection-valued measure]
\label{def: projection_valued_measure}
    Let $H$ be a Hilbert space and let $P(H)$ denote the set of orthogonal projections in $\mathcal{L}(H)$. A finite projection valued measure $\Pi$ on $H$ is a map
    \begin{align*}
        \mathcal{B}(\mathbb{R}) &\to P(H)\\
        A& \mapsto \Pi_A
    \end{align*}
    from the $\sigma$-algebra $\mathcal{B}(\mathbb{R})$ of Borel subsets of $\mathbb{R}$ to the set of projections, such that the following conditions hold:
    \begin{enumerate}
        \item $\Pi_{\emptyset}  =0$, $\Pi_{\mathbb{R}} = \mathrm{Id}$,
        \item For some constant $R>0$, we have $\Pi_{[-R,R]} = \mathrm{Id}$,
        \item If $A_n, n\geq 1$ is an arbitrary sequence of pairwise disjoint Borel subsets of $\mathbb{R}$, let
        \begin{align*}
            A = \bigcup_{n\geq 1}A_n \in  \mathcal{B}(\mathbb{R}),
        \end{align*}
        and then we have $\Pi_{A} = \sum_{n\geq 1}\Pi_{A_n}$, where the series converges in the strong operator topology of $H$.
    \end{enumerate}
\end{defn}

We state the following theorem in terms of projection-valued measures:
\begin{thm}[Spectral theorem for bounded self-adjoint operators]
\label{thm: spectral_general}
    Let $H$ be a separable Hilbert space and let $T = T^{\ast}$ be a bounded self-adjoint operator on $H$.
    Recall that the spectrum $\sigma(T)$ of a self-adjoint operator is defined as the set
    \begin{align*}
        \sigma(T) \doteq \{ \lambda \in \mathbb{R}\mid \lambda \mathrm{Id} - T \text{ is not invertible}.\}
    \end{align*}
    There exists a unique projection-valued measure $\Pi_T$ such that
    \begin{align*}
        T = \int_{\sigma(T)}\lambda \;d\Pi_T(\lambda).
    \end{align*}
    If $f$ is any continuous function on $\sigma(T)$, we have
    \begin{align*}
        f(T) = \int_{\sigma(T)}f(\lambda)d\Pi_T(\lambda).
    \end{align*}
\end{thm}

\begin{lem}[Positive square root]
\label{lem: positive_square_root}
    Let $H$ be a separable Hilbert space and let $A=A^{\ast}\in \mathcal{L}(H)$ be positive. Then there is a unique positive operator $A^{1/2}\in \mathcal{L}(H)$ such that $(A^{1/2})^2=A$.
\end{lem}

\begin{proof}
    Since $A$ is positive, $\sigma(A)\subset [0,\infty)$. By Theorem \ref{thm: spectral_general}, define
    \begin{align*}
        A^{1/2}\doteq\int_{\sigma(A)}\sqrt{\lambda}\;d\Pi_A(\lambda).
    \end{align*}
    Then $A^{1/2}$ is positive and, by the continuous functional calculus,
    \begin{align*}
        (A^{1/2})^2=\int_{\sigma(A)}\lambda\;d\Pi_A(\lambda)=A.
    \end{align*}
    If $B$ is another positive square root of $A$, then $B=B^{\ast}$, $\sigma(B)\subset [0,\infty)$, and
    \begin{align*}
        A^{1/2}=(B^2)^{1/2}=\left(\lambda\mapsto \sqrt{\lambda^2}\right)(B)=B
    \end{align*}
    by the continuous functional calculus. Hence the positive square root is unique.
\end{proof}

\begin{thm}[Polar decomposition]
\label{thm: polar_decomposition}
    Let $H_1, H_2$ be separable Hilbert spaces and let $T : H_1\to H_2$ be a bounded operator. Then there is a unique partial isometry $U\in \mathcal{L}(H_1,H_2)$ such that
    \begin{align*}
        T=U|T|,\qquad |T|\doteq(T^{\ast}T)^{1/2},\qquad \kernel(U)=\kernel(T).
    \end{align*}
    The initial space of $U$ is $\overline{\range(|T|)}=\kernel(T)^{\perp}$ and the final space is $\overline{\range(T)}$.
\end{thm}

\begin{proof}
    By Lemma \ref{lem: positive_square_root}, $|T|=(T^{\ast}T)^{1/2}$ exists. For every $x\in H_1$,
    \begin{align*}
        \||T|x\|^2=\langle |T|^2x,x\rangle=\langle T^{\ast}Tx,x\rangle=\|Tx\|^2.
    \end{align*}
    Hence $\kernel(|T|)=\kernel(T)$. Define
    \begin{align*}
        U_0:\range(|T|)&\to \range(T),\\
        |T|x&\mapsto Tx.
    \end{align*}
    This is well-defined by $\kernel(|T|)=\kernel(T)$, and it is an isometry by the identity above. Therefore $U_0$ extends uniquely to an isometry
    \begin{align*}
        V:\overline{\range(|T|)}\to \overline{\range(T)}.
    \end{align*}
    Since $|T|=|T|^{\ast}$,
    \begin{align*}
        H=\overline{\range(|T|)}\oplus \kernel(|T|).
    \end{align*}
    Define $U$ by
    \begin{align*}
        U(x+y)\doteq Vx,\qquad x\in \overline{\range(|T|)},\quad y\in \kernel(|T|).
    \end{align*}
    Then $U$ is a partial isometry with initial space $\overline{\range(|T|)}$ and final space $\overline{\range(T)}$. Moreover,
    \begin{align*}
        U|T|x=Tx
    \end{align*}
    for every $x\in H$, so $T=U|T|$, and $\kernel(U)=\kernel(|T|)=\kernel(T)$.

    If $W$ is another such partial isometry, then $W|T|=T=U|T|$, so $W=U$ on $\range(|T|)$, hence on $\overline{\range(|T|)}$. Since both vanish on $\kernel(|T|)$, we get $W=U$ on $H_1$.
\end{proof}
\begin{cor}
\label{cor: equiv_hpdag_source}
    Under \Cref{assum:betasource}, there exists $w_P\in \mathcal{H}$ such that
    \begin{align*}
        h_P^{\dag} = \left(T_P^{\ast}T_P\right)^{\frac{\beta_h}{2}}w_P.
    \end{align*}
\end{cor}
\begin{proof}
    By \Cref{assum:betasource}, there exists $v_P\in \mathcal{G}$ such that
    \begin{align*}
        r_{P,\parallel} = \left(T_PT_P^{\ast}\right)^{\frac{\beta_h+1}{2}}v_P.
    \end{align*}We write the polar decomposition of $T_P = U|T_P|$ by \Cref{thm: polar_decomposition}. By \Cref{thm: spectral_general}, we have $\left(T_PT_P^{\ast}\right)^s = \left(U|T_P||T_P|U^{\ast}\right)^s = U\left(T_P^{\ast}T_P\right)^s U^{\ast}$ for any $s>0$.
    Recall that $h_P^{\dag}$ is the minimum norm solution of
    \begin{align*}
        T_P h = r_{P, \parallel} = \left(T_PT_P^{\ast}\right)^{\frac{\beta_h+1}{2}}v_P = U|T_P|\left(T_P^{\ast}T_P\right)^{\frac{\beta_h}{2}}\underbrace{U^{\ast}v_P}_{=: w_P}.
    \end{align*}
    Since $h_P^{\dag}\in \kernel(T_P)^{\perp}$ and $\left(T_P^{\ast}T_P\right)^{\frac{\beta_h}{2}}w_P\in \kernel(T_P)^{\perp}$, the conclusion follows from $h_P^{\dag} - \left(T_P^{\ast}T_P\right)^{\frac{\beta_h}{2}}w_P \in \kernel(T_P)\cap \kernel(T_P)^{\perp}$.
\end{proof}

\subsection{Proofs for \Cref{sub:minimax_learners_of_h_p_dag}}
\begin{proof}[Proof of \cref{lem:bias-tikhonov}]
We adapt the proof of \citet[Lemma 3]{bennett2023sourceconditiondoublerobust} and generalize to general bounded operator $T$ (without assuming compactness). For any $h\in\Hcal_P$ that satisfies the normal equation, we have\begin{align}\label{eq:normal}T_P^*(T_Ph-r_P)=0.\end{align}
Therefore $T_P h - r_P \in \kernel(T_P^{\ast}) = \overline{\range(T_P)}^{\perp}$, hence
\begin{align}
    \label{eq: normal_orth}
    (\forall h\in \mathcal{H}_P)(\forall g\in \overline{\range(T_P)}),\quad \langle T_P h -r_P, g\rangle = 0.
\end{align}
By adding the constant $-\|T_P h_P^\dag-r_P\|^2$ to the objective defining $h_P^\lambda$ in \cref{eqn:populatoin_objective}, we find
\begin{align*}
    h_P^{\lambda} &= \argmin_{h\in \mathcal{H}}\lambda \|h\|^2_{\mathcal{H}} + \left\langle T_P(h-h_P^{\dag}), T_P(h+h_P^{\dag}) - 2r_P\right\rangle_{L^2(P_Z)}\\
    &= \argmin_{h\in \mathcal{H}} \lambda \|h\|^2_{\mathcal{H}} + \left\|T_P(h-h_P^{\dag})\right\|_{L^2(P_Z)}^2 + 2\underbrace{\left \langle T_P(h-h_P^{\dag}), T_Ph_P^{\dag} - r_P\right\rangle_{L^2(P_Z)}}_{=0},
\end{align*}
where we use \cref{eq: normal_orth}. Taking the Fréchet derivative of the above objective with respect to $h$, we have
\begin{align*}
    h^{\lambda}_P = \left(T_P^{\ast}T_P + \lambda \mathrm{Id}_{\mathcal{H}}\right)^{-1}T_P^{\ast}T_Ph_P^{\dag}.
\end{align*}
By \Cref{cor: equiv_hpdag_source} and \Cref{thm: spectral_general}, we have
\begin{align*}
    h_P^{\lambda} - h_P^{\dag} = -\lambda \left(T_P^{\ast}T_P + \lambda \mathrm{Id}_{\mathcal{H}}\right)^{-1}h_P^{\dag} = \left(\int_{\sigma(T_P^{\ast}T_P)}\frac{-\lambda\mu^{\frac{\beta_h}{2}}}{\mu + \lambda}d\Pi_{T_P^{\ast}T_P}(\mu)\right)w_P,
\end{align*}
for some $w_P\in \mathcal{H}$. Hence, by \Cref{thm: spectral_general},
    \begin{align*}
        \left\|h_P^{\lambda} - h_P^{\dag}\right\|_{L^2(P_X)}^2
        &\leq
        \sup_{\mu\in \sigma(T_P^{\ast}T_P)}
        \left(\frac{\lambda \mu^{\frac{\beta_h}{2}}}{\mu+\lambda}\right)^2
        \|w_P\|_{L^2(P_X)}^2.
    \end{align*}
    Since $T_P^{\ast}T_P$ is positive, $\sigma(T_P^{\ast}T_P)\subset [0,\infty)$. The remainder of the proof follows as in the proof of \citet[Lemma 3]{bennett2023sourceconditiondoublerobust}.
\end{proof}

\begin{proof}[Proof of \cref{thm:minimaxlearner}]
Let any $h\in\Hcal,g\in\overline{\range(T_P)}$ be given. Then
\begin{align*}
\E_P[\widetilde{m}(W;g)-h(X)g(Z)]
&=\ip{r_P-T_Ph}{g}\\
&=\ip{r_P-T_Ph_P^\dag}{g}+\ip{T_P(h_P^\dag-h)}{g}\\
&=\ip{T_P(h_P^\dag-h)}{g},
\end{align*}
where the last equality is by \cref{eq:normal}. This gives the equivalent of Eq.(10) in \citet{bennett2023sourceconditiondoublerobust}. The proof then follows the proof of \citet[Theorem 4]{bennett2023sourceconditiondoublerobust}, using \cref{lem:bias-tikhonov} instead of \citet[Lemma 3]{bennett2023sourceconditiondoublerobust}.
\end{proof}

\subsection{Proofs for \Cref{sub:minimax_learners_of_the_weak_norm_riesz_representer}}
We first prove two preparatory results and an algebraic identity, and then provide the proof for \Cref{thm:minimaxlearner_weak_riesz_feasible}. 

\begin{proof}[Proof of \Cref{lem:ppl_reg_bias_alpha}]
The proof follows an analogous calculation to \Cref{lem:bias-tikhonov}. 
We introduce the shorthand
\begin{align*}
    b_{\alpha^h,w}(\lambda)&\doteq \|T_P(\alpha_{P,\lambda}^h-\alpha_P^{h,\dag})\|_{L^2(P_Z)}^2\\
    b_{\alpha^h,s}(\lambda)&\doteq \|\alpha_{P,\lambda}^h-\alpha_P^{h,\dag}\|_{L^2(P_X)}^2.
\end{align*}
Since $a_{P,\parallel} \in \mathrm{range}(T_P^{\ast}T_P)$, the minimum norm element $\alpha_P^{h, \dag} $ that satisfies 
\begin{align*}
    T_P \alpha_P^{h, \dag} = g_{P}^{\dag}
\end{align*}
exists, and it satisfies $\alpha_P^{h, \dag} \in \kernel(T_P)^{\perp}$ by virtue of being the minimum norm element. We have,
\begin{align*}
    \alpha_{P,\lambda}^{h} &= \left(T_P^{\ast}T_P + \lambda \mathrm{Id}_{\mathcal{H}}\right)^{-1}T_P^{\ast}T_P \alpha_P^{h, \dag}\\
    \alpha_{P,\lambda}^{h} - \alpha_{P}^{h, \dag} &= -\lambda \left(T_P^{\ast}T_P + \lambda \mathrm{Id}_{\mathcal{H}}\right)^{-1}\alpha_P^{h, \dag}.
\end{align*}
By \Cref{thm: spectral_general}, applied to the positive self-adjoint operator $T_P^{\ast}T_P$, we have
\begin{align*}
    \alpha_{P,\lambda}^{h}-\alpha_P^{h,\dag}
    =\left(\int_{\sigma(T_P^{\ast}T_P)}\frac{-\lambda}{\mu+\lambda}\;d\Pi_{T_P^{\ast}T_P}(\mu)\right)\alpha_P^{h,\dag}.
\end{align*}
Thus
\begin{align*}
    b_{\alpha^h,s}(\lambda)
    &=\left\|\alpha_{P,\lambda}^{h}-\alpha_P^{h,\dag}\right\|_{L^2(P_X)}^2
    =\int_{\sigma(T_P^{\ast}T_P)}\frac{\lambda^2}{(\mu+\lambda)^2}\;d\left\langle \Pi_{T_P^{\ast}T_P}(\mu)\alpha_P^{h,\dag},\alpha_P^{h,\dag}\right\rangle_{L^2(P_X)},\\
    b_{\alpha^h,w}(\lambda)
    &=\left\|T_P(\alpha_{P,\lambda}^{h}-\alpha_P^{h,\dag})\right\|_{L^2(P_Z)}^2
    =\int_{\sigma(T_P^{\ast}T_P)}\frac{\mu\lambda^2}{(\mu+\lambda)^2}\;d\left\langle \Pi_{T_P^{\ast}T_P}(\mu)\alpha_P^{h,\dag},\alpha_P^{h,\dag}\right\rangle_{L^2(P_X)},
\end{align*}
where the second equality follows from applying \Cref{thm: spectral_general} to the continuous functions $\mu\mapsto \lambda^2/(\mu+\lambda)^2$ and $\mu\mapsto \mu\lambda^2/(\mu+\lambda)^2$. 

We now take $\lambda = \lambda_n \downarrow 0$. We apply dominated convergence theorem with respect to the measure space $\sigma(T_P^{\ast}T_P)$, the scalar-valued measure 
\begin{align*}
    B\mapsto \left\langle \Pi_{T_P^{\ast}T_P}(B)\alpha_P^{h,\dag},\alpha_P^{h,\dag}\right\rangle_{L^2(P_X)},
\end{align*}
and the sequence of functions $\sigma(T_P^{\ast}T_P)\to \mathbb{R}$ given by
\begin{align*}
    \mu\mapsto \frac{\lambda_n^2}{(\mu + \lambda_n)^2}.
\end{align*}
We denote the scalar-valued measure by $\upsilon$. The key step is the following. Since $\alpha_P^{h,\dag}\in\kernel(T_P)^\perp$, the measure $\upsilon$ has no mass at $\mu=0$. For all $\mu\in \sigma(T_P^{\ast}T_P)$ and $\mu\neq 0$, we have $\frac{\lambda^2}{(\mu+\lambda)^2} \to 0$. Moreover, we have 
\begin{align*}
    \left|\frac{\lambda_n^2}{(\mu + \lambda_n)^2}\right| \leq 1.
\end{align*}
Hence, by dominated convergence theorem, we have $b_{\alpha^h,s}(\lambda_n) \to 0$ as $n\to \infty$. 

On the other hand, by another application of dominated convergence theorem, we have
\begin{align}
\label{eq: balpha_h_w_lambda_zero}
    \frac{b_{\alpha^h,w}(\lambda_n)}{\lambda_n}
    =\int_{\sigma(T_P^{\ast}T_P)}\frac{\mu\lambda_n}{(\mu+\lambda_n)^2}\;d\left\langle \Pi_{T_P^{\ast}T_P}(\mu)\alpha_P^{h,\dag},\alpha_P^{h,\dag}\right\rangle_{L^2(P_X)}\to 0,
\end{align}
as $n\to \infty$, since $\mu\lambda/(\mu+\lambda)^2\to 0$ for every $\mu\geq 0$ and $0\leq \mu\lambda/(\mu+\lambda)^2\leq 1/4$. Therefore
\begin{align*}
    \frac{b_{\alpha^h,w}(\lambda_n)}{\lambda_n} = o(1).
\end{align*}
\end{proof}

We now present a concentration of measure result, which will be used in the proof of \Cref{thm:minimaxlearner_weak_riesz_feasible}. 

\begin{lem}[Localized concentration for the $\alpha^h$-learner]
\label{lem:conc_alpha_feasible}
Fix $\lambda=\lambda_{\alpha^h,n}$ and write
\[
    \alpha_\lambda \doteq \alpha_{P,\lambda}^h,
    \qquad
    \alpha_0 \doteq \alpha_P^{h,\dagger}.
\]
Condition on the first-stage sample used to construct $\widehat g$. Suppose
$\widehat g(Z)$ is almost surely bounded by a constant $M<\infty$, and suppose
that
\[
    \sup_{\alpha\in\mathcal H_n}\|\alpha\|_\infty\leq M,
    \qquad
    \sup_{g\in\mathcal G_n}\|g\|_\infty\leq M,
    \qquad
    \|\alpha_\lambda\|_\infty\leq M .
\]
Assume also that
\[
    \alpha_0,\alpha_\lambda\in\mathcal H_n,
    \qquad
    0\in\mathcal G_n,
    \qquad
    T_P(\alpha_0-\alpha)\in\mathcal G_n,
    \qquad \forall \alpha\in\mathcal H_n.
\]
For $\alpha\in\mathcal H_n$, define
\[
    q_\alpha \doteq T_P(\alpha_0-\alpha).
\]
Let $\delta_n$ satisfy
\[
    \delta_n = \Omega\left(\sqrt{\frac{\log\log n+\log(1/\zeta)}{n}}\right),
\]
and suppose that, conditionally on the first-stage sample, $\delta_n$ is a
common upper bound on the critical radii of
\[
    \mathcal G_n,\qquad
    \mathcal H_n\cdot\mathcal G_n,\qquad
    \mathcal H_n-\alpha_\lambda .
\]
Then, after enlarging $\delta_n$ by a universal constant depending only on $M$,
there exists an event $\mathcal E$ with conditional probability at least
$1-\zeta$ such that, on $\mathcal E$, the following inequalities hold
simultaneously:
\begin{align}
\label{eq:conc_alpha_q}
\left|(\mathbb E_n-\mathbb E)\left[
2\widehat g(Z)q_\alpha(Z)-2\alpha(X)q_\alpha(Z)-q_\alpha(Z)^2
\right]\right|
&\le
\delta_n\|q_\alpha\|_{L^2(P_Z)}+\delta_n^2,
&&\forall \alpha\in\mathcal H_n,\\
\label{eq:conc_g_alpha_lambda}
\left|(\mathbb E_n-\mathbb E)\left[
2\widehat g(Z)g(Z)-2\alpha_\lambda(X)g(Z)-g(Z)^2
\right]\right|
&\le
\delta_n\|g\|_{L^2(P_Z)}+\delta_n^2,
&&\forall g\in\mathcal G_n,\\
\label{eq:conc_alpha_sq}
\left|(\mathbb E_n-\mathbb E)\left[
\alpha(X)^2-\alpha_\lambda(X)^2
\right]\right|
&\le
\delta_n\|\alpha-\alpha_\lambda\|_{L^2(P_X)}+\delta_n^2,
&&\forall \alpha\in\mathcal H_n.
\end{align}
\end{lem}
\begin{proof}
Throughout the proof we condition on the first-stage sample, so that
$\widehat g$ is fixed. We apply \Cref{lem:lipschitz_wainwright} three times.

For \cref{eq:conc_alpha_q}, define the vector-valued class
\[
    \mathcal V_1
    \doteq
    \left\{
    W\mapsto
    \big(\alpha(X)q_\alpha(Z),q_\alpha(Z)\big):
    \alpha\in\mathcal H_n
    \right\},
    \qquad
    q_\alpha=T_P(\alpha_0-\alpha).
\]
Its first coordinate projection satisfies
\[
    \mathcal V_1|_1
    =
    \{\alpha q_\alpha:\alpha\in\mathcal H_n\}
    \subseteq
    \mathcal H_n\cdot\mathcal G_n,
\]
because $q_\alpha\in\mathcal G_n$. Its second coordinate projection satisfies
\[
    \mathcal V_1|_2
    =
    \{q_\alpha:\alpha\in\mathcal H_n\}
    \subseteq
    \mathcal G_n.
\]
Moreover, since $\alpha_0\in\mathcal H_n$ and
$q_{\alpha_0}=T_P(\alpha_0-\alpha_0)=0$, the reference element $(0,0)$ belongs to
$\mathcal V_1$.
Apply \Cref{lem:lipschitz_wainwright} with reference element $(0,0)$ and loss function
\[
    \ell_1(u_1,u_2;W)\doteq 2\widehat g(Z)u_2-2u_1-u_2^2 .
\]
The loss $\ell_1$ is Lipschitz in $(u_1,u_2)$ on the restricted bounded range.
Hence, with conditional probability at least $1-\zeta/3$,
\[
\left|(\mathbb E_n-\mathbb E)\left[
2\widehat g(Z)q_\alpha(Z)-2\alpha(X)q_\alpha(Z)-q_\alpha(Z)^2
\right]\right|
\lesssim
\delta_n
\left\|(\alpha q_\alpha,q_\alpha)\right\|_{L^2(P),2}
+\delta_n^2 .
\]
Since $\alpha$ is uniformly bounded, we have
\[
    \left\|(\alpha q_\alpha,q_\alpha)\right\|_{L^2(P),2}
    \lesssim
    \|q_\alpha\|_{L^2(P_Z)}.
\]
Absorbing the constant into $\delta_n$ gives \cref{eq:conc_alpha_q}.

For \cref{eq:conc_g_alpha_lambda}, define
\[
    \mathcal V_2
    \doteq
    \left\{
    W\mapsto
    \big(\alpha_\lambda(X)g(Z),g(Z)\big):
    g\in\mathcal G_n
    \right\}.
\]
Its first coordinate projection satisfies
\[
    \mathcal V_2|_1
    =
    \{\alpha_\lambda g:g\in\mathcal G_n\}
    \subseteq
    \mathcal H_n\cdot\mathcal G_n,
\]
because $\alpha_\lambda\in\mathcal H_n$, and its second coordinate projection is
$\mathcal V_2|_2=\mathcal G_n$. Since $0\in\mathcal G_n$, the reference element
$(0,0)$ belongs to $\mathcal V_2$. Applying \Cref{lem:lipschitz_wainwright} to $\mathcal V_2$ with
reference $(0,0)$ and the loss function $\ell_1$ yields, with (conditional)
probability at least $1-\zeta/3$,
\[
\left|(\mathbb E_n-\mathbb E)\left[
2\widehat g(Z)g(Z)-2\alpha_\lambda(X)g(Z)-g(Z)^2
\right]\right|
\lesssim
\delta_n
\left\|(\alpha_\lambda g,g)\right\|_{L^2(P),2}
+\delta_n^2 .
\]
Since $\alpha_\lambda$ is uniformly bounded, we have $\left\|(\alpha_\lambda g,g)\right\|_{L^2(P),2} \lesssim \|g\|_{L^2(P_Z)}$. 
Absorbing the constant into $\delta_n$ gives \cref{eq:conc_g_alpha_lambda}.

For \cref{eq:conc_alpha_sq}, apply \Cref{lem:lipschitz_wainwright} to the
scalar class $\mathcal H_n$ with reference $\alpha_\lambda$ and loss
$\ell_2(u)=u^2$. The coordinatewise function class is $\mathcal H_n-\alpha_\lambda$, 
whose critical radius is assumed to be bounded by $\delta_n$. Since $u\mapsto
u^2$ is Lipschitz on the bounded range of $\mathcal H_n$, with (conditional)
probability at least $1-\zeta/3$,
\[
\left|(\mathbb E_n-\mathbb E)\left[
\alpha(X)^2-\alpha_\lambda(X)^2
\right]\right|
\lesssim
\delta_n\|\alpha-\alpha_\lambda\|_{L^2(P_X)}
+\delta_n^2 .
\]
Absorbing the constant into $\delta_n$ gives \cref{eq:conc_alpha_sq}. The desired conclusion follows by taking a union bound over the failure probabilities of the three events. 
\end{proof}

\paragraph{Algebraic identity.} We have the identity, for all $\alpha\in \Hcal_n$, we have
    \begin{equation}
        \label{eq:identity}
        \begin{aligned}
            &\left\|T_P\alpha - g_P^{\dag}\right\|^2_{2} + \lambda \|\alpha\|^2_{2} - \left\|T_P \alpha_{\lambda} - g_P^{\dag}\right\|_2^2 - \lambda \|\alpha_{\lambda}\|^2_{2}\\
        &= \|T_P(\alpha- \alpha_{\lambda})\|^2_{2} + \lambda \|\alpha - \alpha_{\lambda}\|^2_{2}.
        \end{aligned}
    \end{equation}
    To show this identity, recall
\begin{align*}
    \alpha_{P,\lambda}^h &= \argmin_{\alpha\in \mathcal{H}} \lambda \|\alpha\|^2_{L^2(P_X)} + \|T_P\alpha - g_P^{\dag}\|^2_{L^2(P_Z)}.
\end{align*}
By computing the first order Fréchet derivative, we find
\begin{align}
\label{eq: frechet_zero_lambda_alpha_Plambda}
    \lambda \alpha_{P,\lambda}^h + T_P^{\ast}T_P\alpha_{P,\lambda}^h - T_P^{\ast}g_P^{\dag} = 0.
\end{align}
Therefore, for any $\alpha\in \mathcal{H}$, we have
\begin{align*}
    &\left\|T_P\alpha - g_P^{\dag}\right\|^2_{L^2(P_Z)} + \lambda \|\alpha\|^2_{L^2(P_X)} - \left\|T_P\alpha_{P,\lambda}^h - g_P^{\dag}\right\|^2_{L^2(P_Z)} - \lambda \|\alpha^h_{P,\lambda}\|^2_{L^2(P_X)}\\
    &= \left\|T_P (\alpha - \alpha_{P,\lambda}^h)\right\|^2_{L^2(P_Z)} + \lambda \left\|\alpha - \alpha_{P,\lambda}^h\right\|^2_{L^2(P_X)}\\
    &+\underbrace{2 \left\langle \alpha - \alpha_{P,\lambda}^h, T_P^{\ast}T_P\alpha_{P,\lambda}^h - T_P^{\ast}g_P^{\dag}\right\rangle_{L^2(P_X)} + 2\lambda \left\langle \alpha - \alpha_{P,\lambda}^h, \alpha_{P,\lambda}^h\right\rangle_{L^2(P_X)}}_{= 0},
\end{align*}
where we use \cref{eq: frechet_zero_lambda_alpha_Plambda}.

\begin{proof}[Proof of \cref{thm:minimaxlearner_weak_riesz_feasible}]
Throughout the proof, we condition on the independent samples used to construct $\widehat{g}(\lambda_g)$. We use the shorthand
\[
\begin{alignedat}{3}
    \widehat g &\doteq \widehat g(\lambda_g), 
    &\qquad \lambda &\doteq \lambda_{\alpha^h,n}, 
    &\qquad \widehat\alpha &\doteq \widehat\alpha^h(\widehat g,\lambda),\\
    \alpha_P &\doteq \alpha_P^{h,\dag}, 
    &\qquad \alpha_\lambda &\doteq \alpha_{P,\lambda}^h, 
    &\qquad \delta_n &\doteq \delta_{\alpha^h,n}.
\end{alignedat}
\]
We use $\|\cdot\|_2$ to denote relevant $L^2$-norms. We use $\|\cdot\|_{2,n}$ to denote relevant $L^2(P_n)$-norms.

We recall the definition of $\widehat{\alpha}$:
    \begin{align}
    \label{eq:defn_alpha_hat}
        \widehat{\alpha} = \argmin_{\alpha\in \Hcal_n}\max_{g\in \Gcal_n} \mathbb{E}_n\left[2\widehat{g}(Z)g(Z) - 2\alpha(X)g(Z) - g(Z)^2 + \lambda \alpha(X)^2\right].
    \end{align}
    We introduce another useful shorthand
    \begin{align*}
        \widetilde{g} \doteq T_P(\alpha_P - \widehat{\alpha}).
    \end{align*}
    We have
    \begin{align*}
        &\left\|T_P(\alpha_P - \widehat{\alpha})\right\|^2_2\\
        &= 2\lra{g_P^{\dag}, \widetilde{g}}_2 - 2\lra{ T_P\widehat{\alpha}, \widetilde{g}}_2 - \lra{ \widetilde{g}, \widetilde{g}}_2\\
        &= 2\lra{\widehat{g}, \widetilde{g}}_2 + 2\lra{g_P^{\dag} - \widehat{g}, \widetilde{g}}_2 - 2\lra{T_P\widehat{\alpha} , \widetilde{g}}_2 - \lra{\widetilde{g}, \widetilde{g}}_2\\
        &= 2\lra{g_P^{\dag} - \widehat{g}, \widetilde{g}}_2 + \mathbb{E}\left[2\widehat{g}(Z)\widetilde{g}(Z) - 2\widehat{\alpha}(X)\widetilde{g}(Z) - \widetilde{g}(Z)^2\right],
    \end{align*}
    where in the last step we use the property of orthogonal projection $T_P$. Since $\widetilde g=q_{\widehat\alpha}$, \cref{eq:conc_alpha_q} gives, on $\mathcal E$,
    \begin{align}
        \left\|T_P(\alpha_P - \widehat{\alpha})\right\|^2_2 &\leq \delta_n \left\|\widetilde{g}\right\|_{2} + \delta_n^2 + 2\lra{g_P^{\dag} - \widehat{g}, \widetilde{g}}_2 \tag{A}\label{eq:decomp_a_hatalpha}\\
        &+ \mathbb{E}_n\left[2\widehat{g}(Z) \widetilde{g}(Z) - 2 \widehat{\alpha}(X)\widetilde{g}(Z) - {\widetilde{g}(Z)}^2\right] \tag{B}\label{eq:decomp_b_hatalpha}.
    \end{align}
    First, by the closedness \Cref{assum:wellspecified_weakriesz}, we have $\widetilde{g}\in \Gcal_n$. We have, on $\mathcal E$,
    \begin{align*}
        \cref{eq:decomp_b_hatalpha} &\leq \sup_{g\in \Gcal_n}\E_n\left[2\widehat{g}(Z)g(Z) - 2\widehat{\alpha}(X)g(Z) - g(Z)^2\right]\\
        &\stackrel{(a)}{\leq} \sup_{g\in \Gcal_n}\E_n\left[2\widehat{g}(Z)g(Z) - 2\alpha_{\lambda}(X)g(Z) - g(Z)^2 + \lambda \alpha_{\lambda}(X)^2 - \lambda \widehat{\alpha}(X)^2\right]\\
        &\stackrel{(b)}{\leq} \sup_{g\in \Gcal_n}\left\{\E\left[2\widehat{g}(Z)g(Z) - 2\alpha_{\lambda}(X)g(Z) - g(Z)^2 \right] + \delta_n \|g\|_{2} + \delta_n^2\right\} + \E_n\left[\lambda \alpha_{\lambda}(X)^2 - \lambda \widehat{\alpha}(X)^2\right]\\
        &\leq \sup_{g\in \Gcal_n}\E\left[2\widehat{g}(Z)g(Z) - 2\alpha_{\lambda}(X)g(Z) - \frac{1}{2}g(Z)^2\right] + \frac{3\delta_n^2}{2} + \E_n\left[\lambda \alpha_{\lambda}(X)^2 - \lambda \widehat{\alpha}(X)^2\right]\\
        &\leq \sup_{g\in L^2(P_Z)}\E\left[2\widehat{g}(Z)g(Z) - 2\alpha_{\lambda}(X)g(Z) - \frac{1}{2}g(Z)^2\right] + \frac{3\delta_n^2}{2} + \E_n\left[\lambda \alpha_{\lambda}(X)^2 - \lambda \widehat{\alpha}(X)^2\right]\\
        &\stackrel{(c)}{=} 2\left\|\widehat{g} - T_P\alpha_{\lambda}\right\|^2_{2} + \frac{3\delta_n^2}{2} + \E_n\left[\lambda \alpha_{\lambda}(X)^2 - \lambda \widehat{\alpha}(X)^2\right].
    \end{align*}
    In the above derivations, we use \cref{eq:defn_alpha_hat} in $(a)$, we use \cref{eq:conc_g_alpha_lambda} in $(b)$, in $(c)$, after enlarging the supremum to $L^2(P_Z)$, the supremum is attained at $g=2(\widehat g-T_P\alpha_\lambda)$. We have, by differences of squares, 
    \begin{align}
    \label{eq:diff_sq_alpha_id}
        \left\|\widehat{g} - T_P\alpha_{\lambda}\right\|^2_2 - \left\|g_P^{\dag} - T_P\alpha_{\lambda}\right\|^2_{2} = \lra{\widehat{g} + g_P^{\dag} - 2T_P\alpha_{\lambda}, \widehat{g} - g_P^{\dag}}_2.
    \end{align}
    Therefore,
    \begin{align}
        &\left\langle g_P^{\dag} - \widehat{g}, \widetilde{g}\right\rangle_2 + \left\|\widehat{g} - T_P\alpha_{\lambda}\right\|^2_2\nonumber\\
        &= \left\langle g_P^{\dag} - \widehat{g}, \widetilde{g}\right\rangle_2 + \left\|g_P^{\dag} - T_P\alpha_{\lambda}\right\|^2_{2} + \lra{\widehat{g} + g_P^{\dag} - 2T_P\alpha_{\lambda}, \widehat{g} - g_P^{\dag}}_2\nonumber\\
        &= \left\langle g_P^{\dag} - \widehat{g}, \widetilde{g} + 2T_P\alpha_{\lambda} - \widehat{g} - g_P^{\dag}\right\rangle_2 + \left\|g_P^{\dag} - T_P\alpha_{\lambda}\right\|^2_{2}\nonumber\\
        &= \left\langle g_P^{\dag} - \widehat{g}, \cancel{T_P\alpha_P} - T_P\widehat{\alpha} + 2T_P\alpha_{\lambda} - \widehat{g} - \cancel{g_P^{\dag}}\right\rangle_2 + \left\|g_P^{\dag} - T_P\alpha_{\lambda}\right\|^2_{2}\nonumber\\
        &= \left\langle g_P^{\dag} - \widehat{g},  - T_P\widehat{\alpha} + 2T_P\alpha_{\lambda} - \widehat{g}\right\rangle_2 + \left\|g_P^{\dag} - T_P\alpha_{\lambda}\right\|^2_{2}\label{eq:diff_sq_alpha_id_v2}.
    \end{align}
    Hence, on $\mathcal E$,
    \begin{align}
        &\|T_P(\alpha_P - \widehat{\alpha})\|^2_{2}\nonumber\\
        &\leq \delta_n \left\|\widetilde{g}\right\|_{2} + \frac{5}{2}\delta_n^2 + 2\lra{g_P^{\dag} - \widehat{g}, \widetilde{g}}_2 + 2\left\|\widehat{g} - T_P\alpha_{\lambda}\right\|^2_{2} +  \E_n\left[\lambda \alpha_{\lambda}(X)^2 - \lambda \widehat{\alpha}(X)^2\right]\nonumber\\
        &\stackrel{(a)}{=} \delta_n \left\|\widetilde{g}\right\|_{2} + \frac{5}{2}\delta_n^2 + 2\lra{g_P^{\dag} - \widehat{g},  - T_P\widehat{\alpha} + 2T_P\alpha_{\lambda} - \widehat{g}}_2+ 2\left\|T_P(\alpha_P - \alpha_{\lambda})\right\|^2_{2}\nonumber\\ &\qquad  + \E_n\left[\lambda \alpha_{\lambda}(X)^2 - \lambda \widehat{\alpha}(X)^2\right]\nonumber\\
        &\stackrel{(b)}{=} 2\left\|T_P(\alpha_P - \alpha_{\lambda})\right\|^2_{2} + \frac{5}{2}\delta_n^2 + \E_n\left[\lambda \alpha_{\lambda}(X)^2 - \lambda \widehat{\alpha}(X)^2\right] + \delta_n \left\|\widetilde{g}\right\|_{2}\nonumber\\
        &\qquad + 2\lra{g_P^{\dag} - \widehat{g},  T_P(\alpha_{\lambda} - \widehat{\alpha})}_2 + 2\lra{g_P^{\dag} - \widehat{g}, T_P(\alpha_{\lambda} - \alpha_P)}_2 + 2\lra{g_P^{\dag} - \widehat{g},  g_P^{\dag} - \widehat{g}}_{2}\nonumber\\
        &\stackrel{(c)}{\leq} 3\left\|T_P(\alpha_P - \alpha_{\lambda})\right\|^2_{2} + 3\left\|g_P^{\dag} - \widehat{g}\right\|^2_{2} + \frac{5}{2}\delta_n^2 + \E_n\left[\lambda \alpha_{\lambda}(X)^2 - \lambda \widehat{\alpha}(X)^2\right] + \delta_n \left\|\widetilde{g}\right\|_{2}\nonumber\\
        &\qquad + 2\lra{2\left(g_P^{\dag} - \widehat{g}\right),  \frac{1}{2}\left(T_P(\alpha_{\lambda} - \widehat{\alpha})\right)}_2\nonumber\\
        &\stackrel{(d)}{\leq}3\left\|T_P(\alpha_P - \alpha_{\lambda})\right\|^2_{2} + 7\left\|g_P^{\dag} - \widehat{g}\right\|^2_{2} + \frac{5}{2}\delta_n^2 \nonumber\\ &\qquad + \E_n\left[\lambda \alpha_{\lambda}(X)^2 - \lambda \widehat{\alpha}(X)^2\right] + \delta_n \left\|\widetilde{g}\right\|_{2} + \frac{1}{4}\left\|T_P(\alpha_{\lambda} - \widehat{\alpha})\right\|^2_{2} , \label{eq:tp_alphap_alphahat}
    \end{align}
    In the above derivations, we use \cref{eq:diff_sq_alpha_id_v2} in $(a)$, in $(b)$ we split the term $\left\langle g_P^{\dag} - \widehat{g}, 2T_P\alpha_{\lambda} - T_P\widehat{\alpha} - \widehat{g}\right\rangle_2$ in three terms, recovering the squared norm $2\left\|g_P^{\dag} - \widehat{g}\right\|^2_{2}$, in $(c)$ we complete the squares to see
    \begin{align*}
        \|T_P(\alpha_P - \alpha_{\lambda})\|^2_{2} + \left\|g_P^{\dag} - \widehat{g}\right\|^2_{2} + 2 \lra{g_P^{\dag} - \widehat{g}, T_P(\alpha_{\lambda} - \alpha_P)}_2\geq 0,
    \end{align*}and in $(d)$ we further complete the squares to find 
    \begin{align*}
        2\lra{2\left(g_P^{\dag} - \widehat{g}\right),  \frac{1}{2}\left(T_P(\alpha_{\lambda} - \widehat{\alpha})\right)}_2 \leq 4\left\|g_P^{\dag} - \widehat{g}\right\|^2_{2} + \frac{1}{4}\left\|T_P(\alpha_{\lambda} - \widehat{\alpha})\right\|^2_{2}. 
    \end{align*}
     We find, on $\mathcal E$,
    \begin{align*}
        &\|T_P(\widehat{\alpha}- \alpha_{\lambda})\|^2_{2} + \lambda \|\widehat{\alpha} - \alpha_{\lambda}\|^2_{2}\\
        &\stackrel{(a)}{=} \|T_P (\widehat{\alpha} - \alpha_P)\|^2_{2} + \lambda \|\widehat{\alpha}\|^2_2 - \|T_P(\alpha_{\lambda} - \alpha_P)\|^2_{2} - \lambda \|\alpha_{\lambda}\|^2_{2} \\
        &\stackrel{(b)}{\leq} 2\left\|T_P(\alpha_P - \alpha_{\lambda})\right\|^2_{2} + 7\left\|g_P^{\dag} - \widehat{g}\right\|^2_{2} + \frac{5}{2}\delta_n^2\\ &\qquad +\lambda (\mathbb{E}- \E_n)\left[ \widehat{\alpha}(X)^2- \alpha_{\lambda}(X)^2\right] + \delta_n \left\|\widetilde{g}\right\|_{2} + \frac{1}{4}\left\|T_P(\alpha_{\lambda} - \widehat{\alpha})\right\|^2_{2}\\
        &\stackrel{(c)}{\leq} 2\left\|T_P(\alpha_P - \alpha_{\lambda})\right\|^2_{2} + 7\left\|g_P^{\dag} - \widehat{g}\right\|^2_{2} + \frac{9}{2}\delta_n^2\\ &\qquad +\lambda (\mathbb{E}- \E_n)\left[ \widehat{\alpha}(X)^2- \alpha_{\lambda}(X)^2\right] + \frac{1}{8}\|T_P(\alpha_P - \widehat{\alpha})\|^2_{2} + \frac{1}{4}\left\|T_P(\alpha_{\lambda} - \widehat{\alpha})\right\|^2_{2}\\
        &\stackrel{(d)}{\leq} \frac{9}{4}\left\|T_P(\alpha_P - \alpha_{\lambda})\right\|^2_{2} + 7\left\|g_P^{\dag} - \widehat{g}\right\|^2_{2} + \frac{9}{2}\delta_n^2\\ &\qquad +\lambda (\mathbb{E}- \E_n)\left[ \widehat{\alpha}(X)^2- \alpha_{\lambda}(X)^2\right] + \frac{1}{2}\left\|T_P(\alpha_{\lambda} - \widehat{\alpha})\right\|^2_{2},
    \end{align*}
    where in the above derivations, we apply the identity \cref{eq:identity} for $\widehat{\alpha}$ in $(a)$, we apply \cref{eq:tp_alphap_alphahat} in $(b)$, we complete the square in $(c)$ to write
    \begin{align*}
        \delta_n \|\widetilde{g}\|_{2}\leq \frac{1}{8}\|T_P(\alpha_P - \widehat{\alpha})\|^2_{2} + 2\delta_n^2,
    \end{align*}
    and we apply the triangular inequality to $\|T_P(\alpha_P - \widehat{\alpha})\|^2_{2}$ in $(d)$. By \cref{eq:conc_alpha_sq} with $\widehat\alpha$, on $\mathcal E$, 
    \begin{align*}
    (\mathbb{E}_{n}- \mathbb{E})\left[\alpha_{\lambda}(X)^2 - \widehat{\alpha}(X)^2\right] \leq \delta_{n}\left\|\alpha_{\lambda} - \widehat{\alpha}\right\|_{2} + \delta_{n}^2 \leq \frac{\delta_{n}^2}{2} + \frac{1}{2}\left\|\alpha_{\lambda} - \widehat{\alpha}\right\|_{2}^2 + \delta_{n}^2.
\end{align*}
Hence on $\mathcal E$: 
\begin{align*}
    &\|T_P(\widehat{\alpha}- \alpha_{\lambda})\|^2_{2} + \lambda \|\widehat{\alpha} - \alpha_{\lambda}\|^2_{2}\\
    &\leq \frac{9}{4}\left\|T_P(\alpha_P - \alpha_{\lambda})\right\|^2_{2} + 7\left\|g_P^{\dag} - \widehat{g}\right\|^2_{2} + \frac{9}{2}\delta_n^2\\ &\qquad +
    \frac{\lambda}{2}\delta_n^2 + \frac{\lambda}{2}\|\alpha_{\lambda} - \widehat{\alpha}\|^2_{2} + \lambda \delta_n^2 + \frac{1}{2}\left\|T_P(\alpha_{\lambda} - \widehat{\alpha})\right\|^2_{2}. 
\end{align*}
Rearranging and using $\lambda \leq 1$, we have
\begin{align*}
    \|T_P(\widehat{\alpha}- \alpha_{\lambda})\|^2_{2} + \lambda \|\widehat{\alpha} - \alpha_{\lambda}\|^2_{2} \leq \frac{9}{2}\left\|T_P(\alpha_P - \alpha_{\lambda})\right\|^2_{2} + 14\left\|g_P^{\dag} - \widehat{g}\right\|^2_{2} + 12\delta_n^2.  
\end{align*}
By \Cref{lem:ppl_reg_bias_alpha}, we have on $\mathcal{E}$, 
\begin{align*}
    \left\|T_P(\widehat{\alpha} - \alpha_P)\right\|^2_{2} &\leq 20\|T_P(\alpha_P - \alpha_{\lambda})\|^2_{2} + 28\left\|g_P^{\dag} - \widehat{g}\right\|^2_{2} + 24\delta_n^2\\
    &= O\left(\lambda + \delta_n^2 + \left\|g_P^{\dag} - \widehat{g}\right\|^2_{2}\right).
\end{align*}
We also have
\begin{align*}
     \|\widehat{\alpha} - \alpha_P\|^2_{2} \leq 2\|\alpha_{\lambda} - \alpha_P\|^2_{2} + \frac{2}{\lambda}\left\{9\left\|T_P(\alpha_P - \alpha_{\lambda})\right\|^2_{2} + 14\left\|g_P^{\dag} - \widehat{g}\right\|^2_{2} + 12\delta_n^2\right\}. 
\end{align*}
Since $\left\|T_P(\alpha_P - \alpha_{\lambda})\right\|^2_{2} = o(\lambda)$ and $\|\alpha_{\lambda} - \alpha_P\|^2_{2} = o(1)$, the right hand side is $o(1)$ if 
\begin{align*}
    \frac{2}{\lambda}\left\{14\left\|g_P^{\dag} - \widehat{g}\right\|^2_{2} + 12\delta_n^2\right\} = o(1),
\end{align*}
as $n\to\infty$. 
\end{proof}

\subsection{Proofs for \Cref{sub:learners_of_orthogonal_residuals}}
\begin{proof}[Proof of \cref{thm:orthogonal_residual_fixed}]
We define the shorthand
\begin{align*}
    q_0 &\doteq r_P-T_Ph_1\\
    \widehat q &\doteq \widehat r_\perp(h_1)\\
    \ell_q(W) &\doteq q(Z)^2-2\big\{\widetilde m(W;q)-h_1(X)q(Z)\big\}. 
\end{align*}
We have,
\begin{align*}
    &\E_P[\ell_q(W)-\ell_{q_0}(W)]\\
    &=\E_P\left[q(Z)^2 - q_0(Z)^2 - 2\{\widetilde{m}(W;q) - h_1(X)q(Z)\} + 2\{\widetilde{m}(W;q_0) - h_1(X)q_0(Z)\}\right]\\
    &= \E_P \left[q(Z)^2 - q_0(Z)^2 - 2r_P(Z)q(Z) + 2r_P(Z)q_0(Z) + 2h_1(X)(q(Z) - q_0(Z))\right]\\
    &= \E_P \left[q(Z)^2 - q_0(Z)^2 - 2q_0(Z)q(Z) + 2(r_P(Z) - T_Ph_1(Z))q_0(Z)\right]\\
    &= \|q - q_0\|^2_{L^2(P_Z)}. 
\end{align*}
By \Cref{eq:rperp_learner}, $\widehat r_\perp(h_1)$ is a minimizer of  $\E_n\ell_q(W)$ over $\Gcal_n^{r}$, namely
\begin{align*}
    \widehat{q} = \argmin_{q\in \Gcal_{r,n}} \E_{n}[\ell_q(W)]. 
\end{align*}
Hence we have
\begin{align*}
    \|\widehat{q} - q_0\|_{L^2(P_Z)}^2 &= \E_P\left[\ell_{\widehat{q}}(W) - \ell_{q_0}(W)\right]\\
    &= \underbrace{\E_{n}\left[\ell_{\widehat{q}}(W) - \ell_{q_0}(W)\right]}_{\leq 0 \text{ as }q_0\in \Gcal_{r,n}} + (\E_P - \E_{n})\left[\ell_{\widehat{q}}(W) - \ell_{q_0}(W)\right]\\
    &\leq \big|(\E_P - \E_{n})\left[\ell_{\widehat{q}}(W) - \ell_{q_0}(W)\right]\big|. 
\end{align*}
Thus it suffices to bound $\big|(\E_P - \E_{n})\left[\ell_{\widehat{q}}(W) - \ell_{q_0}(W)\right]\big|$. Since $q(Z)$ is uniformly bounded by $O(1)$ over $\Gcal_n^{r}$, and $h_1(X)$ is uniformly bounded, we have that $q(Z)^2 + 2h_1(X)q(Z)$ is $O(1)$-Lipschitz in $q(Z)$. On the other hand, \Cref{assum:mscont} yields
\[
    \|\widetilde m(W;q)-\widetilde m(W;q_0)\|_{L^2(P)}
    \lesssim
    \|q-q_0\|_{L^2(P_Z)}.
\]
Thus \Cref{lem:lipschitz_wainwright} gives, uniformly over $q\in\Gcal_{r,n}$, with probability at least $1-\zeta$,
\[
    \left|(\E_n-\E_P)[\ell_q(W)-\ell_{q_0}(W)]\right| =
    O\left(\delta_{r_\perp,n}\|q-q_0\|_{L^2(P_Z)}+ \delta_{r_\perp,n}^2\right).
\]
Setting $q = \widehat{q}$ in the above inequality, we obtain the following inequality for $\|\widehat{q} - q_0\|_{L^2(P_Z)}$
\begin{align*}
   \|\widehat{q} - q_0\|^2_{L^2(P_Z)} \leq \left|(\E_n-\E_P)[\ell_{\widehat{q}}(W)-\ell_{q_0}(W)]\right| =
    O\left(\delta_{r_\perp,n}\|\widehat{q}-q_0\|_{L^2(P_Z)}+ \delta_{r_\perp,n}^2\right).
\end{align*}
Solving this inequality yields
\[
    \|\widehat r_\perp(h_1)-(r_P-T_Ph_1)\|_{L^2(P_Z)}^2 = O(\delta_{r_\perp,n}^2).
\]
\end{proof}

\begin{proof}[Proof of \cref{cor:rperp_plugin}]
We condition on the first-stage sample used to build $\widehat h(\lambda_h)$. By
\Cref{thm:orthogonal_residual_fixed},
\[
    \|\widehat r_\perp(\widehat h(\lambda_h))-
    \{r_P-T_P\widehat h(\lambda_h)\}\|_{L^2(P_Z)}^2
    =
    O_p(\delta_{r_\perp,n}^2).
\]
Since $r_{P,\perp}=r_P-T_Ph_P^\dag$, \Cref{thm:minimaxlearner} gives
\[
    \|T_P(\widehat h(\lambda_h)-h_P^\dag)\|_{L^2(P_Z)}^2
    =
    O_p\left(\delta_{h,n}^2+
    \lambda_h^{\min\{\beta_h+1,2\}}\right).
\]
The result follows from applying the triangle inequality and then marginalizing over the
first-stage sample. 
\end{proof}

\paragraph{Notation.}
For a vector-valued class $\mathcal F\subseteq\{f:\mathcal X\to\mathbb R^d\}$,
write
\[
    \mathcal F|_t:=\{f_t:f\in\mathcal F\},\qquad t=1,\ldots,d,
\]
where $f_t$ is the $t$-th coordinate of $f$. For $f:\mathcal X\to\mathbb R^d$,
set
\[
    \|f\|_{2,2}:=\left(\mathbb E\|f(X)\|_{\ell_2}^2\right)^{1/2}. 
\]
For a scalar class $\mathcal A$, define
\[
    \operatorname{star}(\mathcal A):=\{t a:a\in\mathcal A,\ 0\le t\le1\}, \qquad \mathcal A-a^\star:=\{a-a^\star:a\in\mathcal A\}.
\]
We write $\mathcal R(\mathcal A,r)$ when $n$ is clear. 
The following lemma is \citet[Lemma 14]{foster2023orthogonal}. 
\begin{lem}[Localized concentration]\label{lem:lipschitz_wainwright}
Consider a function class $\cF$, with $\sup_{f\in \cF} \|f\|_{\infty}\leq 1$, and pick any $f^{\star}\in \cF$. Let $\delta_n^2\geq \frac{4\,d\,\log(41\log(2c_2 n))}{c_2 n}$ be any solution to the inequalities:
\begin{equation}
\forall t\in \{1,\ldots, d\}: \mathcal{R}(\ensuremath{\mathrm{star}}(\mathcal{F}|_t - f_t^{\star}),\delta) \leq \delta^2.
\end{equation}
Moreover, assume that the loss $\ell$ is $L$-Lipschitz in its first argument with respect to the $\ell_2$ norm. Then for some universal constants $c_5, c_6$, with probability $1-c_5 \exp(-c_6 n \delta_n^2)$,
\begin{align}\label{eqn:lipschitz_loss_gen}
\left| \mathbb{P}_n (\mathcal{L}_{f} - \mathcal{L}_{f^{\star}}) - \mathbb{P} (\mathcal{L}_{f} - \mathcal{L}_{f^{\star}}) \right| \leq~& 18 L\, d\, \delta_n \{ \|f - f^{\star}\|_{2,2} + \delta_n\},\quad \forall f \in \cF.
\end{align}
\end{lem}


\section{Proofs for debiased inference}\label{sec:inference_proofs}

\paragraph{Empirical process notations.} Let $P_n$ denote the empirical measure of $\mathcal{I}_4$, i.e. the sample used for evaluation in the debiased estimator $\widehat{\Psi}$. For a function $f: E_W \to \mathbb{R}$, we define
\begin{align*}
    P_n f &\doteq \frac{1}{n}\sum_{w_i\in \mathcal{I}_4} f(w_i)\\
    P f &\doteq \int f(w)\,P(dw).
\end{align*}

\begin{proof}[Proof of \cref{thm:autodml-clt}]
Recall $\widehat{\Psi}$ defined in \cref{eq:autodml-estimator}
\begin{align*}
    \widehat{\Psi} &= \frac{1}{n}\sum_{i\in \mathcal{I}_4} \chi(W_i;\widehat{\eta})\\ &= P_n \left\{\chi(\cdot;\widehat{\eta})\right\} \\
    &= (P_n - P)\left\{\chi(\cdot;\widehat{\eta})\right\} + P\left\{\chi(\cdot;\widehat{\eta})\right\} \\
    &= (P_n - P)\left\{\chi(\cdot;\widehat{\eta}) - \chi_P\right\} + (P_n - P)\left\{\chi_P\right\} + P\left\{\chi(\cdot;\widehat{\eta})\right\}, 
\end{align*}
where $\chi_P$ is defined in \cref{eq:chi_P_defn}. We therefore have the algebraic expansion 
\begin{align*}
\sqrt n(\widehat\Psi-\Psi(P))
&=
\sqrt n(P_n-P)\{\chi_P\}\tag{IF}\label{eq:vme_if}\\
&\quad+
\sqrt n(P_n-P)\left\{\chi(\cdot;\widehat\eta)-\chi_P\right\}\tag{EQ}\label{eq:vme_eq}\\
&\quad+
\sqrt n\left(P\{\chi(\cdot;\widehat\eta)\}-\Psi(P)\right).\tag{R}\label{eq:vme_r}
\end{align*}
Since $\chi_P(w) = \varphi_P(w) + \Psi(P)$, and $\mathbb{E}_P[\chi_P(W)] = \Psi(P)$, we have
\begin{align*}
    \eqref{eq:vme_if} &= \sqrt{n}\left(\frac{1}{n}\sum_{i\in \mathcal{I}_4}\varphi_P(W_i)\right).
\end{align*}
We will leave \eqref{eq:vme_if} as is and defer applying a Central Limit Theorem towards the end. 
We now control \eqref{eq:vme_eq} using the fact that $\mathcal{I}_4$ is independent of $\mathcal{I}_1, \mathcal{I}_2, \mathcal{I}_3$. In the absence of such an independence assumption, we would need to restrict the complexity of nuisance function class and invoke a uniform law of large numbers \citep{dudley2014uniform}. 
By independence, the expectation over the samples in $\mathcal{I}_4$ is the same as the conditional expectation over $\mathcal{I}_4$ conditioned on $\mathcal{I}_1, \mathcal{I}_2, \mathcal{I}_3$. We denote it by $\mathbb{E}_{\mathcal{I}_4}$.  We have
\begin{align*}
    &\mathbb{E}_{\mathcal{I}_4}\left[\bigg((P_n-P)\left\{\chi(\cdot;\widehat\eta)-\chi_P\right\}\bigg)^2\right] = \frac1n
\Var_P\{\chi(W;\widehat\eta)-\chi_P(W)\} \leq \frac{1}{n} \mathbb{E}_{W}\left[\bigg(\chi(W;\widehat{\eta}) - \chi_P(W)\bigg)^2\right]. 
\end{align*}
Thus conditional Chebyshev inequality gives, for every $\epsilon>0$,
\begin{align}
    \label{eq:cond_chebyshev}
    \Prb\left(
\left|\sqrt n(P_n-P)\{\chi(\cdot;\widehat\eta)-\chi_P\}\right|>\epsilon
\,\middle|\,\widehat\eta
\right)
\le
\frac{
\mathbb{E}_{W}\big[(\chi(W;\widehat{\eta}) - \chi_P(W))^2\big]
}{\epsilon^2}.
\end{align}
We have
\begin{align*}
    &\chi(w;\widehat{\eta}) - \chi_P\\
    &= m(w;\widehat{h}) + \widetilde{m}(w;\widehat{g}) - \widehat{h}(x)\widehat{g}(z) + \widehat{r}_{\perp}(z)\{\widehat{\alpha}^h(x) - \widehat{g}(z)\} + \widehat{a}_{\perp}(x)\{\widehat{\alpha}^g(z) - \widehat{h}(x)\}\\
    &\quad -  m(w; h_P^{\dag}) - \widetilde{m}(w;g_P^{\dag}) + h_P^{\dag}(x)g_P^{\dag}(z) - r_{P, \perp}(z) \left\{ \alpha_P^{h,\dag}(x) - g_P^{\dag}(z)\right\} - a_{P, \perp}(x)\left\{\alpha_{P}^{g,\dag}(z) - h_P^{\dag}(x)\right\}.
\end{align*}
We have, by the above identity and triangle inequality, 
\begin{align}
    &\mathbb{E}_{W}\left[\big(\chi(W;\widehat{\eta}) - \chi_P(W)\big)^2\right]\nonumber\\ 
    &\leq 5\mathbb{E}_W\left[\left(m(W;\widehat{h}) - m(W;h_P^{\dag})\right)^2\right] +5\mathbb{E}_W\left[\left(\widetilde{m}(W;\widehat{g}) - \widetilde{m}(W;g_P^{\dag})\right)^2\right] \tag{A}\label{eq:inf_bded_a}\\
    &+5\mathbb{E}_W\left[\left(\widehat{h}(X)\widehat{g}(Z) - h_P^{\dag}(X)g_P^{\dag}(Z)\right)^2\right] \tag{B}\label{eq:inf_bded_b}\\
    &+5 \mathbb{E}_W\left[\left(\widehat{r}_{\perp}(Z)\{\widehat{\alpha}^h(X) - \widehat{g}(Z)\} - r_{P, \perp}(Z) \left\{ \alpha_P^{h,\dag}(X)- g_P^{\dag}(Z)\right\}\right)^2\right]\tag{C}\label{eq:inf_bded_c}\\
    &+ 5\mathbb{E}_W\left[\left(\widehat{a}_{\perp}(X)\{\widehat{\alpha}^g(Z) - \widehat{h}(X)\} - a_{P, \perp}(X)\left\{\alpha_{P}^{g,\dag}(Z) - h_P^{\dag}(X)\right\}\right)^2\right].\tag{D}\label{eq:inf_bded_d}
\end{align}
By linearity of $m$ and $\widetilde{m}$,  \Cref{assum:mscont}, we have
\begin{align*}
    \eqref{eq:inf_bded_a} &\leq 5C\left\|\widehat{h} - h_P^{\dag}\right\|^2_{L^2(P_X)} + 5C\left\|\widehat{g} - g_P^{\dag}\right\|_{L^2(P_Z)}^2 \\
    &= 5C(e^{s}_h)^2 +5C(e^{s}_g)^2. 
\end{align*}
We next use the identity $a_1 b_1 - a_2b_2 = a_2(b_1 - b_2) + b_1(a_1 -a_2)$ to find
\begin{align*}
    \left(a_1 b_1 - a_2b_2\right)^2 \leq 2 a_2^2(b_1 - b_2)^2 + 2b_1^2 (a_1-a_2)^2. 
\end{align*}
We have
\begin{align*}
    \eqref{eq:inf_bded_b} &\leq 10\left\|\widehat{h}\right\|_{L^{\infty}(P_X)}^2 \left\|\widehat{g} - g^{\dag}_P\right\|_{L^2(P_Z)}^2 + 10 \left\|g^{\dag}_P\right\|^2_{L^{\infty}(P_Z)}\left\|\widehat{h} - h_P^{\dag}\right\|^2_{L^2(P_X)}\\
    \eqref{eq:inf_bded_c} &\leq 20 \left\|\widehat{r}_{\perp}\right\|^2_{L^{\infty}(P_Z)} \left(\left\|\widehat{\alpha}^h - \alpha_P^{h,\dag}\right\|_{L^2(P_X)}^2 +  \left\|\widehat{g}  -g_P^{\dag}\right\|^2_{L^2(P_Z)}\right)\\ &+ 20\left(\left\|\alpha_P^{h,\dag}\right\|_{L^{\infty}(P_X)}^2 + \left\|g_P^{\dag}\right\|_{L^{\infty}(P_Z)}^2\right)\left\|\widehat{r}_{\perp} - r_{P,\perp}\right\|^2_{L^2(P_Z)}\\
    \eqref{eq:inf_bded_d} &\leq 20 \left\|\widehat{a}_{\perp}\right\|^2_{L^{\infty}(P_X)} \left(\left\|\widehat{\alpha}^g - \alpha_P^{g,\dag}\right\|_{L^2(P_Z)}^2 +  \left\|\widehat{h}  -h_P^{\dag}\right\|^2_{L^2(P_X)}\right)\\ &+ 20\left(\left\|\alpha_P^{g,\dag}\right\|_{L^{\infty}(P_Z)}^2 + \left\|h_P^{\dag}\right\|_{L^{\infty}(P_X)}^2\right)\left\|\widehat{a}_{\perp} - a_{P,\perp}\right\|^2_{L^2(P_X)}\\
\end{align*}
Under the assumption that there exists a large constant $C>0$ such that
\begin{align*}
    \max\left\{\left\|\widehat{h}\right\|_{\infty}^2,\left\|h^{\dag}_P\right\|^2_{\infty}, \left\|g^{\dag}_P\right\|^2_{\infty}, \left\|\alpha_P^{h,\dag}\right\|_{\infty}^2, \left\|\alpha_P^{g,\dag}\right\|_{\infty}^2, \left\|\widehat{a}_{\perp}\right\|^2_{\infty},\left\|\widehat{r}_{\perp}\right\|^2_{\infty} \right\} \leq C,
\end{align*}
we have, upon redefining $C$, 
\begin{align*}
   \mathbb{E}_{W}\left[\big(\chi(W;\widehat{\eta}) - \chi_P\big)^2\right] \leq C\left\{(e^{s}_h)^2  + (e^{s}_g)^2 + (e^s_{\alpha^h})^2 + (e^s_{\alpha^g})^2 + (e^{s}_{r_{\perp}})^2 + (e^{s}_{a_{\perp}})^2\right\}. 
\end{align*}
As $n\to \infty$, all the terms on the right hand side are assumed to converge to $0$ in probability by \Cref{cond:inference_nuisance}, where the randomness is induced by $\mathcal{I}_1, \mathcal{I}_2, \mathcal{I}_3$. For any $\delta>0$, we define the event
\begin{align*}
    B_{\delta} \doteq \left\{\mathbb{E}_{W}\left[\big(\chi(W;\widehat{\eta}) - \chi_P(W)\big)^2\right] > \delta\right\}
\end{align*}
with respect to probability measure on samples $\mathcal{I}_i$, $i=1,2,3$. Let $1_{B_{\delta}}$ denote the indicator function of $B_{\delta}$. For every $\epsilon > 0$, we have
\begin{align*}
    P\left(|\eqref{eq:vme_eq}| \geq \epsilon\right) &= \E \left[P(|\eqref{eq:vme_eq}| > \epsilon \mid \widehat{\eta})1_{B_{\delta}}\right] + \E \left[P(|\eqref{eq:vme_eq}| > \epsilon \mid \widehat{\eta})(1- 1_{B_{\delta}})\right]\\
    &\leq P(B_{\delta}) + \frac{\delta}{\epsilon^2},
\end{align*}
where we apply \cref{eq:cond_chebyshev}. We have shown that $\lim_{n\to \infty} P(B_{\delta}) = 0$ for any fixed $\delta$. Hence, for every $\delta>0$, we have
\begin{align*}
    \limsup_{n\to \infty} P\left(|\eqref{eq:vme_eq}| \geq \epsilon\right) \leq \frac{\delta}{\epsilon^2}.
\end{align*}
Hence $|\eqref{eq:vme_eq}| = o_p(1)$. 

We bound \eqref{eq:vme_r} using the functional bias expansion derived in \Cref{thm: bias_char}. From \cref{eq:simple_bias_expansion}, we have
\begin{align*}
    n^{-\frac{1}{2}}|\eqref{eq:vme_r}| &\leq \left|\left\langle
    T_P(\widehat{h}-h_P^\dag),\widehat{g}-g_P^\dag
    \right\rangle_{L^2(P_Z)}\right|\\
    &+ \left|\left\langle
    T_P\widehat{\alpha}^h-\widehat{g},\widehat{r}_\perp-r_{P,\perp}
    \right\rangle_{L^2(P_Z)}\right|\\
    &+ \left|\left\langle
    T_P^*\widehat{\alpha}^g-\widehat{h},\widehat{a}_\perp-a_{P,\perp}
    \right\rangle_{L^2(P_X)}\right|.
\end{align*}
The first term above is bounded by $\min\{e_g^w e_h^s,\ e_g^s e_h^w\}$, by \Cref{rem:rates_weaken}. We further have, by the triangle inequality and the fact that $T_P \alpha_{P}^{h,\dag} = g_P^{\dag}$ and $ T_P^{\ast}\alpha_P^{g,\dag} = h_P^{\dag}$, 
\begin{align*}
    \left\|T_P\widehat\alpha^h-\widehat g\right\|_{L^2(P_Z)} &\leq \left\|T_P(\widehat{\alpha}^h - \alpha_{P}^{h,\dag})\right\|_{L^2(P_Z)} + \left\|g_P^{\dag} - \widehat{g}\right\|_{L^2(P_Z)} = e^w_{\alpha^h} + e^{s}_g\\
    \left\|T_P^{\ast}\widehat{\alpha}^g - \widehat{h}\right\|_{L^2(P_X)} &\leq  \left\|T_P^{\ast}(\widehat{\alpha}^g - \alpha_P^{g,\dag})\right\|_{L^2(P_X)} + \left\|h_P^{\dag}  - \widehat{h}\right\|_{L^2(P_X)} = e_{\alpha^g}^w + e_h^{s}
\end{align*}
Recalling
\begin{align*}
    e^{s}_{r_{\perp}} = \left\|\widehat{r}_{\perp} - r_{P, \perp}\right\|_{L^2(P_Z)}, \quad e^{s}_{a_{\perp}} = \left\|\widehat{a}_{\perp} - a_{P, \perp}\right\|_{L^2(P_X)},
\end{align*}
we obtain by applying Cauchy-Schwarz inequality that
\begin{align*}
n^{-\frac{1}{2}}\left|\eqref{eq:vme_r}\right| \leq &
\min\{e_g^w e_h^s,\ e_g^s e_h^w\} +
(e_{\alpha^h}^w+e_g^s)e_{r_\perp}^s
+
(e_{\alpha^g}^w+e_h^s)e_{a_\perp}^s .
\end{align*}
We further use \cref{eq:product_err_1} and \cref{eq:product_err_2} to conclude that
\begin{align*}
    \left|\eqref{eq:vme_r}\right| = o_p(1). 
\end{align*}
Therefore we have shown that
\begin{align}
\label{eq:negl_rem}
    |\eqref{eq:vme_eq} + \eqref{eq:vme_r}| = o_p(1)
\end{align}
under the Theorem's assumptions. This yields the asymptotic linearity of our estimator $\widehat{\Psi}$, as stated in \cref{eq:asymp_linear}. By assumption, $\sigma_P^2 = \E_P [\varphi_P(W)^2] < \infty$. Therefore we may apply a standard central limit theorem to conclude that 
\begin{align*}
    \eqref{eq:vme_if} \rightsquigarrow \mathcal{N}(0, \sigma_P^2). 
\end{align*}
We further have, 
\begin{align*}
    \sqrt{n}\left(\widehat{\Psi} - \Psi(P)\right) \rightsquigarrow \mathcal{N}(0, \sigma_P^2),
\end{align*}
due to \cref{eq:negl_rem} and an application of Slutsky's theorem. 

It remains to show variance consistency. We have already shown that
$P\{\chi(\cdot;\widehat\eta)-\chi_P\}^2=o_p(1)$. Conditional on the training
samples, $\widehat\eta$ is fixed, so
\[
\E_{\mathcal I_4}\left[\frac1n \sum_{i\in \mathcal{I}_4}\left(\chi(W_i;\widehat\eta)- \chi_P(W_i)\right)^2\,\middle|\,\widehat\eta\right]
=
\mathbb{E}_{W}\left[\big(\chi(W;\widehat{\eta}) - \chi_P(W)\big)^2\right] = o_p(1).
\]
Hence by Markov's inequality, 
\begin{align*}
    P_n\{\chi(\cdot;\widehat{\eta}) - \chi_P\}^2 = \frac1n\sum_{i\in \mathcal{I}_4}\left(\chi(W_i;\widehat\eta)-\chi_P(W_i)\right)^2 =  o_p(1). 
\end{align*}
We apply the reverse triangle inequality with respect to the empirical $L^2(P_n)$-norm, 
\begin{align*}
&\left|
\left[P_n\{\chi(\cdot;\widehat\eta)-P_n\chi(\cdot;\widehat\eta)\}^2\right]^{1/2}-\left[P_n\{\chi_P-P_n\chi_P\}^2\right]^{1/2}\right| \\
&\quad = \left|\left\|\chi(\cdot;\widehat\eta)-P_n\chi(\cdot;\widehat\eta)\right\|_{L^2(P_n)} - \left\|\chi_P-P_n\chi_P\right\|_{L^2(P_n)} \right|\\
&\quad \leq \left\|\chi(\cdot;\widehat\eta)- \chi_P - P_n \left\{\chi(\cdot;\widehat\eta) - \chi_P\right\}\right\|_{L^2(P_n)}\\
&\quad \leq \left\|\chi(\cdot;\widehat{\eta}) - \chi_P\right\|_{L^2(P_n)} = o_p(1),
\end{align*}
where the last inequality uses
\begin{align*}
    P_n\{f-P_nf\}^2=P_nf^2-(P_nf)^2\le P_nf^2. 
\end{align*}
Finally, since $\chi_P\in L^2(P)$, the law of large numbers gives
$P_n\chi_P\to_p P\chi_P$ and $P_n\chi_P^2\to_p P\chi_P^2$. Hence
\begin{align*}
    P_n \{\chi_P-P_n\chi_P\}^2 = P_n\chi_P^2-(P_n\chi_P)^2 \to_p P\chi_P^2-(P\chi_P)^2 = \sigma_P^2. 
\end{align*}
Hence we have proved $\widehat\sigma\to_p\sigma_P$.
\end{proof}

\end{document}